\documentclass{article}
\PassOptionsToPackage{authoryear,round}{natbib}
\usepackage[preprint]{neurips_2026}
\usepackage[utf8]{inputenc}
\usepackage[T1]{fontenc}
\usepackage{amsmath,amsfonts,amsthm,amssymb,bm}
\usepackage{graphicx,subfig,booktabs,multirow,algorithm,algpseudocode,microtype,needspace,placeins}
\usepackage{capt-of}
\usepackage{hyperref,url}
\hypersetup{hidelinks,pdfauthor={Yao Lu, Dengdong Fan, Shixun Zhang, Yonghong Tian},pdftitle={PowerStep: Memory-Efficient Adaptive Optimization via lp-Norm Steepest Descent}}
\newtheorem{assumption}{Assumption}
\newtheorem{theorem}{Theorem}

\newtheorem{lemma}{Lemma}
\newtheorem{corollary}{Corollary}
\title{PowerStep: Memory-Efficient Adaptive Optimization\\via $\ell_p$-Norm Steepest Descent}
\author{%
Yao Lu$^1$\thanks{Email: \texttt{yaolubrain@gmail.com}}\quad
Dengdong Fan$^1$\quad Shixun Zhang$^1$\quad Yonghong Tian$^{1,2}$\\
$^1$Pengcheng Laboratory\qquad $^2$Peking University
}
\makeatletter
\renewcommand{\@noticestring}{}
\makeatother
\begin{document}
\maketitle
\begin{abstract}
Adaptive optimizers such as Adam are standard for training Transformers, but storing gradient first and second moments incurs substantial memory overhead. We introduce PowerStep, a memory-efficient optimizer that achieves coordinate-wise adaptivity without storing second-moment statistics. Motivated by $\ell_p$-norm steepest descent, PowerStep applies a signed-power transform directly to one momentum buffer. We establish a finite-horizon stationarity bound for exact, unregularized updates, with an $O(1/\sqrt{T})$ term and a noise-dependent residual. Experiments on Transformers from 124M to 235B parameters show competitive validation quality while halving $\texttt{fp32}$ optimizer-state memory relative to AdamW. Combined with uniform $\texttt{int8}$ quantization, PowerStep remains numerically stable and reduces optimizer-state memory by $\sim8\times$ compared to $\texttt{fp32}$ AdamW. PowerStep thus provides a simple, memory-efficient alternative for large-scale training. Code is available at \url{https://github.com/yaolubrain/PowerStep}.
\end{abstract}

\section{Introduction}

Neural networks are difficult to train because of vanishing and exploding gradients \citep{hochreiter1991untersuchungen,bengio1994learning} and ill-conditioned loss landscapes \citep{lecun1998efficient,dauphin2014identifying,choromanska2015loss}, which can slow stochastic gradient descent (SGD) or cause divergence. Adaptive gradient methods \citep{duchi2011adaptive,tieleman2012rmsprop,kingma2015adam,loshchilov2019decoupled} precondition updates using accumulated second-moment statistics. Among these methods, Adam and its decoupled weight decay variant AdamW have become default optimizers for large-scale neural networks, particularly Transformers \citep{vaswani2017attention}.

Consider an objective $f(\boldsymbol{\theta}):\mathbb{R}^d\to\mathbb{R}$. Let $\mathbf{g}_t$ denote the stochastic gradient, and let $\mathbf{m}_t$ and $\mathbf{v}_t$ denote its first- and second-moment estimates. Omitting bias correction, Adam updates parameters as
\begin{equation}
\boldsymbol{\theta}_t = \boldsymbol{\theta}_{t-1} - \eta_t \cdot \mathbf{m}_t / (\sqrt{\mathbf{v}_t} + \epsilon),
\end{equation}
where $\eta_t$ is the learning rate and $\epsilon$ is a small constant. Dividing the first moment by the square root of the second moment rescales the momentum coordinate-wise, dampening updates in coordinates with large historical gradients relative to those with small ones. Equivalently, Adam applies a diagonal preconditioner $\mathbf{D}_t=\operatorname{diag}((\sqrt{\mathbf{v}_t}+\epsilon)^{-1})$ to the momentum direction. The empirical benefits of this adaptivity in Transformer training remain under active investigation \citep{balles2018dissecting,pan2022directional,kunstner2023noise,zhang2024why,tomihari2025understanding,zhao2025deconstructing,jin2026adam}.

However, the second-moment accumulator $\mathbf{v}_t$ incurs substantial memory overhead. It has the same dimensionality as the model parameters and is typically stored in \texttt{fp32}, doubling optimizer-state storage relative to SGD with momentum. For a model with $d$ parameters, the two \texttt{fp32} buffers require $8d$ bytes, excluding parameters, gradients and activations. Unlike temporary tensors used during an update, these buffers persist throughout training. The second-moment buffer alone costs four bytes per parameter, making its removal a direct way to reduce persistent memory as models grow. This cost has motivated methods that approximate second-moment statistics through factorization \citep{shazeer2018adafactor} or blockwise sharing \citep{zheng2019blockwise,zhang2025adammini}, quantize optimizer states \citep{dettmers2022bit,li2024memory,han2025qadam}, or dispense with the second-moment buffer entirely \citep{bernstein2018signsgd,balles2018dissecting,chen2024symbolic,zhang2025adams}. Methods such as GaLore~\citep{zhao2024galore} and LDAdam~\citep{robert2024ldadam} reduce optimizer-state memory by projecting gradients into a low-dimensional subspace, maintaining optimizer states and computing updates there, and projecting the updates back to the full parameter space.

A related line of work explores optimization under $\ell_p$-norm geometry. Early work by \citet{grove2001general} and \citet{gentile2003robustness} established convergence guarantees for $\ell_p$-norm algorithms in linear classification and regression. Powerball \citep{yuan2019powerball} uses $\ell_p$-norm steepest descent to update parameters directly, while pbSGDM, the momentum variant of pbSGD \citep{zhou2021pbsgd}, applies a signed-power transform before momentum accumulation. Stacey \citep{luo2025stacey} applies the transform to a momentum buffer within a primal--dual interpolation framework for non-convex optimization.

In this paper, we use $\ell_p$-norm geometry to design a single-state optimizer. Our contributions are as follows:

\begin{itemize}
    \item \textbf{Algorithm design.} We derive PowerStep from $\ell_p$-norm steepest descent and show that applying a signed-power transform directly to a heavy-ball momentum buffer yields coordinate-wise adaptivity without second-moment storage.
    \item \textbf{Theoretical analysis.} We establish a finite-horizon stationarity bound for exact, unregularized updates, with an $O(1/\sqrt{T})$ term and an explicit noise-dependent residual. In the noiseless case, the bound yields $O(1/\sqrt{T})$ stationarity.
    \item \textbf{Empirical validation.} We evaluate dense and mixture-of-experts Transformers from 124M to 235B parameters. PowerStep achieves competitive validation quality while halving \texttt{fp32} optimizer-state memory. Combined with uniform \texttt{int8} quantization, it remains stable in our experiments and reduces state storage by $\sim8\times$ relative to \texttt{fp32} AdamW.
\end{itemize}

\section{Algorithm}
\label{sec:algorithm}

PowerStep stores a single heavy-ball momentum buffer~\citep{polyak1964some} and updates parameters by
\begin{align}
\mathbf{m}_t &= \gamma \mathbf{m}_{t-1} + \mathbf{g}_t, \label{eq:powerstep1} \\
\bm{\theta}_t &= \bm{\theta}_{t-1} - \eta_t \Phi_\beta(\mathbf{m}_t), \label{eq:powerstep2}
\end{align}
where $\gamma\in[0,1)$ is the momentum coefficient controlling the retention of past gradients, and $\beta\in[0,1]$ is the power exponent controlling the compression of momentum magnitudes. The transform $\Phi_\beta(\mathbf x)=\operatorname{sign}(\mathbf x)\odot|\mathbf x|^\beta$ acts elementwise. Algorithm~\ref{alg:psd} includes decoupled weight decay. We next motivate the direction through $\ell_p$-norm geometry.

\subsection{Steepest Descent under \texorpdfstring{$\ell_p$}{lp}-Norm Geometry}
Steepest descent selects an update direction $\mathbf{v}$ that minimizes the first-order approximation to the objective $f(\bm{\theta})$. In the framework of steepest descent in general normed spaces \citep{boyd2004convex}, the optimal direction is defined as
\begin{equation}
\mathbf{v}^* = \arg\min_{\mathbf{v} \in \mathbb{R}^d} \langle \nabla f(\bm{\theta}), \mathbf{v} \rangle \quad \text{s.t. } \|\mathbf{v}\| \leq 1.
\end{equation}
The geometry of the trust region is dictated by the choice of norm. By adopting the Minkowski $\ell_p$ norm, $\|\mathbf{v}\|_p = (\sum_{i=1}^d |v_i|^p)^{1/p}$, we define a continuum of geometries. To derive the update, we form the Lagrangian using the $p$-th power of the norm constraint
\begin{equation}
\mathcal{L}(\mathbf{v}, \mu) = \sum_{i=1}^d g_i v_i + \mu \left( \sum_{i=1}^d |v_i|^p - 1 \right),
\end{equation}
where $g_i$ is the $i$-th component of the gradient $\mathbf{g} = \nabla f(\bm{\theta})$ and $\mu > 0$ is the Lagrange multiplier. Setting $\frac{\partial \mathcal{L}}{\partial v_i} = 0$ yields the coordinate-wise update rule
\begin{align}
v_i^* \propto -\operatorname{sign}(g_i) \cdot |g_i|^{\frac{1}{p-1}}.
\label{eq:v}
\end{align}
For $\mathbf g\ne0$ and $1<p<\infty$, the normalized solution is $\mathbf v^*=-\Phi_\beta(\mathbf g)/\|\mathbf g\|_{1+\beta}^{\beta}$, where $\beta=1/(p-1)$. Powerball~\citep{yuan2019powerball} and PowerStep use the unnormalized signed-power direction; the learning rate determines its length rather than imposing a fixed trust-region radius.

\begin{figure}[!t]
\centering\includegraphics[width=.38\linewidth]{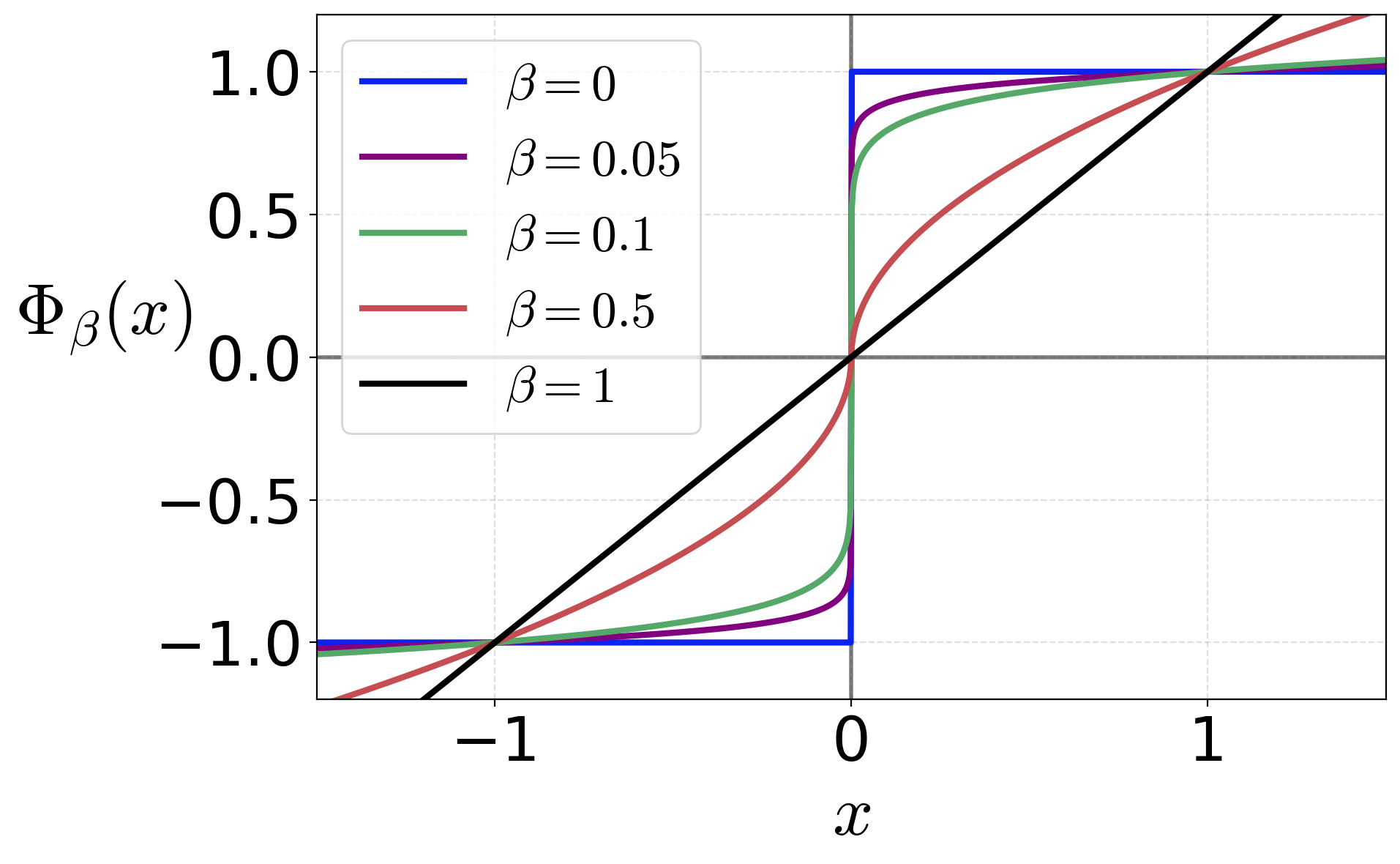}
\caption{Signed-power transform for several exponents. Coordinate-wise magnitude compression interpolates between a sign update and a linear update.}
\label{fig:spt}
\end{figure}

\subsection{Adaptivity via \texorpdfstring{$\ell_p$}{lp}-Norm Steepest Descent}
The transform defined above follows the direction in~\eqref{eq:v} when $\beta=1/(p-1)\in(0,1]$ for $p\in[2,\infty)$. At $\beta=0$, we define $\Phi_0=\operatorname{sign}$ with $\operatorname{sign}(0)=0$.
As illustrated in Figure~\ref{fig:spt}, for $0<\beta<1$, the transform compresses the ratio of two nonzero magnitudes from $r$ to $r^\beta$, reducing magnitude disparities while preserving signs.

PowerStep applies this transform after heavy-ball momentum accumulation~\citep{polyak1964some}, allowing temporal smoothing before nonlinear scaling. Powerball instead transforms raw gradients. For a constant gradient $\mathbf g$, heavy-ball momentum approaches $\mathbf g/(1-\gamma)$, so the signed-power update gains a factor $(1-\gamma)^{-\beta}$. At $(\gamma,\beta)=(0.9,0.1)$ this factor is only $1.26$, providing a scale-based rationale for testing the same learning-rate range as AdamW. The theory studies updates~\eqref{eq:powerstep1}--\eqref{eq:powerstep2} without weight decay or gradient clipping.

Equivalently, the PowerStep update~(\ref{eq:powerstep2}) can be expressed as a preconditioned momentum step,
\begin{align}
\bm{\theta}_t &= \bm{\theta}_{t-1} - \eta_t \mathbf{D}_t \mathbf{m}_t, \\
\mathbf{D}_t &= \operatorname{diag}\bigl(|\mathbf{m}_t|^{\beta-1}\bigr), \label{eq:up1}
\end{align}
where the $i$-th diagonal entry is $|m_{t,i}|^{\beta-1}$ for nonzero $m_{t,i}$. At zero, this expression is singular for $\beta<1$, but the implemented transform $\Phi_\beta(0)=0$ is well defined. For $0<\beta<1$, each entry is inversely related to momentum magnitude. This coordinate-wise scaling is computed directly from $\mathbf m_t$ without a second-moment estimator. 

At the extremes, PowerStep recovers classical methods: $\beta = 0$ yields $\Phi_0(\mathbf{m}_t) = \operatorname{sign}(\mathbf{m}_t)$, reducing to SignSGD with momentum~\citep{bernstein2018signsgd,balles2018dissecting}; $\beta = 1$ yields $\Phi_1(\mathbf{m}_t) = \mathbf{m}_t$, recovering standard SGD with momentum~\citep{polyak1964some}. Between these limits, PowerStep interpolates between sign-based and linear updates, with $\beta = 0.1$ providing a practically effective trade-off (Section~\ref{sec:exp}).

\begin{algorithm}[H]
\caption{PowerStep}\label{alg:psd}
\begin{algorithmic}[1]
\Require Initial parameters $\bm{\theta}_0$, learning rates $\{\eta_t\}_{t=1}^T$
\Require Momentum $\gamma\in[0,1)$, exponent $\beta\in[0,1]$, weight decay $\lambda\ge0$
\State $\mathbf{m}_0\leftarrow\mathbf{0}$
\For{$t=1$ to $T$}
\State $\mathbf{g}_t\leftarrow\nabla f_t(\bm{\theta}_{t-1})$ \Comment{Stochastic gradient}
\State $\mathbf{m}_t\leftarrow\gamma \mathbf{m}_{t-1}+\mathbf{g}_t$ \Comment{Accumulate gradients}
\State $\mathbf{u}_t\leftarrow\operatorname{sign}(\mathbf{m}_t)\odot|\mathbf{m}_t|^\beta$ \Comment{Elementwise signed-power transform}
\State $\bm{\theta}_t\leftarrow(1-\eta_t\lambda)\bm{\theta}_{t-1}-\eta_t\mathbf{u}_t$
\EndFor
\end{algorithmic}
\end{algorithm}

\subsection{Relation to Prior Methods}
Powerball~\citep{yuan2019powerball} applies the signed-power transform to raw gradients, providing no temporal smoothing. pbSGDM~\citep{zhou2021pbsgd} applies it before momentum accumulation. PowerStep instead applies it after momentum accumulation, so the nonlinear scaling reflects the accumulated state. These operations do not commute; Section~\ref{sec:exp} compares their empirical behavior.

Stacey~\citep{luo2025stacey} applies a stabilized transform $\Phi_\beta^\epsilon(\mathbf x)=\operatorname{sign}(\mathbf x)\odot(|\mathbf x|+\epsilon)^\beta$ to a momentum-like state in a primal--dual framework, making it the closest relative. Stacey maintains two state buffers and aims to accelerate optimization through primal--dual interpolation. PowerStep has a different aim: reducing optimizer-state memory. It removes the auxiliary state and stabilization term, retaining only one momentum buffer. We compare PowerStep with Stacey in Appendix~\ref{sec:stacey}.

\section{Convergence Analysis}
\label{sec:convergence}
We analyze the exact, unregularized update ($\lambda=0$), without gradient clipping or state quantization. The guarantee separates a finite-horizon term from a noise-dependent residual. Let $0<\beta\le1$, $q=1+\beta$, $p=q/\beta$, and $h=1-\gamma>0$. These are dual exponents, and
\begin{equation}\label{eq:native-identity}
\langle \mathbf{m},\Phi_\beta(\mathbf{m})\rangle=\|\mathbf{m}\|_q^q,
\qquad \|\Phi_\beta(\mathbf{m})\|_p=\|\mathbf{m}\|_q^\beta.
\end{equation}

\begin{assumption}[Smoothness and boundedness]\label{assum:regularity}
For constants $L_p,G>0$, the continuously differentiable objective is bounded below by $f^*$, satisfies $\|\nabla f(\mathbf{x})\|_q\le G$, and has a Lipschitz gradient in the dual pair:
\begin{equation}
\|\nabla f(\mathbf{x})-\nabla f(\mathbf{y})\|_q\le L_p\|\mathbf{x}-\mathbf{y}\|_p.
\end{equation}
\end{assumption}
\begin{assumption}[Stochastic oracle]\label{assum:oracle}
With $\mathcal F_{t-1}$ the history before sampling $\mathbf{g}_t$, write $\bar{\mathbf{g}}_t=\nabla f(\bm{\theta}_{t-1})$ and $\bm{\xi}_t=\mathbf{g}_t-\bar{\mathbf{g}}_t$. We assume $\mathbb E[\|\bm{\xi}_t\|_q^q\mid\mathcal F_{t-1}]\le\sigma^q$; the oracle may be biased.
\end{assumption}
The $q$-moment condition need not imply a second moment for $\bm{\xi}_t$ when $q<2$. Our proof only bounds the second moment of the \emph{transformed update}, using $2\beta\le q$.

\begin{theorem}[Finite-horizon stationarity bound]\label{thm:convergence}
Under Assumptions~\ref{assum:regularity}--\ref{assum:oracle}, fix a horizon $T$, use $\eta_t=\eta=c/\sqrt T$ with $c>0$, and initialize $\mathbf{m}_0=\mathbf{0}$. Let $\nu_t=\mathbb E\|\bm{\xi}_t\|_q^q$. There are explicit constants $C_1,C_2\ge0$ and $C_\beta>0$, independent of $T$, such that
\begin{equation}\label{eq:corrected-rate}
\frac1T\sum_{t=1}^T\mathbb E\|\bar{\mathbf{g}}_t\|_q^q
\le \underbrace{\frac{C_1}{\sqrt T}+\frac{C_2}{T}}_{R_T}
+\frac{C_\beta}{T}\sum_{t=1}^T\nu_t.
\end{equation}
Here $C_1,C_2$ depend on $L_p,G,\sigma,\gamma,\beta,c$ and the initial objective gap; $C_\beta$ depends only on $\beta$. Appendix~\ref{app:proofs} gives their formulas.
Consequently, $\min_{t\le T}\mathbb E\|\bar{\mathbf{g}}_t\|_2^2\le G^{1-\beta}(R_T+C_\beta\sigma^q)$.
\end{theorem}

The bound gives $O(T^{-1/2})$ stationarity when $\sigma=0$; fixed stochastic noise leaves a residual. The constant step size differs from experimental warmup/cosine schedules. Constants may depend on dimension after conversion from Euclidean norms and become loose as $\beta\to0$; the theorem excludes $\beta=0$ and does not establish an optimal rate. Appendix~\ref{app:proofs} gives the proof, explicit constants, dimension conversions and conditions for a vanishing residual.

\section{Low-Precision State Storage}\label{sec:low_precision}
\citet{dettmers2022bit} quantize optimizer states with block-wise scales and a nonuniform codebook. Our AdamW (Dettmers \texttt{int8}) baseline uses local codebooks that differ from the bitsandbytes implementation. PowerStep instead uses uniform quantization in 128-coordinate blocks for all states, including embeddings. It reconstructs $\hat{\mathbf{m}}_{t-1}$, computes $\mathbf{a}_t=\mathbf{g}_t+\gamma\hat{\mathbf{m}}_{t-1}$ and $\Phi_\beta(\mathbf{a}_t)$ in full precision, then stores $Q(\mathbf{a}_t)$. One byte per coordinate and one \texttt{fp32} scale per block require $1+4/128$ bytes per parameter, versus 8 for \texttt{fp32} AdamW: a $\sim8\times$ reduction.

The signed-power transform has relative condition number $\beta$ away from zero and an absolute H\"older bound at zero. In contrast, rounding AdamW's $v=10^{-8}$ to zero with $\epsilon=10^{-8}$ amplifies $1/(\sqrt v+\epsilon)$ by $10{,}001$. Appendix~\ref{sec:quantization-analysis} gives the sensitivity analysis; Section~\ref{sec:quantization} compares training stability.

\section{Experiments}\label{sec:exp}
We examine whether a single momentum state can preserve optimization quality while reducing memory, and whether its signed-power update remains stable under simple state quantization. We evaluate pretraining, matrix-optimizer hybrids, and mathematical fine-tuning.
\subsection{Setup}
\paragraph{Models and datasets.}
We evaluate Transformer pretraining from 124M to 235B parameters. Dense models include GPT-2 (124M, 350M)~\citep{radford2019language} and Qwen3 (0.6B, 1.7B, 4B, 8B, 32B)~\citep{yang2025qwen3}; MoE models include DeepSeek-V2-Lite (16B total, 2.4B active)~\citep{deepseekai2024deepseekv2}, Qwen3-30B-A3B and Qwen3-235B-A22B. GPT-2 uses OpenWebText~\citep{Gokaslan2019OpenWeb}; the other models use C4~\citep{raffel2020exploring}.

\paragraph{Optimizers and hyperparameters.}
The main comparison uses AdamW, AdamS~\citep{zhang2025adams}, SignSGD with momentum, pbSGDM and PowerStep with \texttt{fp32} states; quantized states are labeled explicitly. AdamW and AdamS use $(\beta_1,\beta_2)=(0.9,0.95)$ and $\epsilon=10^{-8}$; PowerStep and pbSGDM use $(\gamma,\beta)=(0.9,0.1)$. The main recipe uses $\eta_{\max}\in[2,6]\times10^{-4}$, depending on model size, weight decay 0.1, gradient clipping at 1.0, and 2,000-step warmup followed by cosine decay. GPT-2 uses batch 480, context length 1,024 and 50,000 updates; other models use batch 256, context length 2,048 and 20,000 updates. MoE load-balancing loss has coefficient $10^{-3}$. Appendix~\ref{sec:exp_app} gives per-model settings and implementation details.

\paragraph{Infrastructure and precision.}
We use Ascend hardware with nanoGPT\footnote{\url{https://github.com/karpathy/nanoGPT}} for GPT-2 and Megatron-Core\footnote{\url{https://github.com/NVIDIA/Megatron-LM}}/MindSpeed-LLM\footnote{\url{https://gitcode.com/Ascend/MindSpeed-LLM}} for larger models. The large-model recipe uses \texttt{bf16}; the matched nanoGPT study in Appendix~\ref{app:matched} uses \texttt{fp32} parameters with \texttt{bf16} autocast. Muon~\citep{jordan2024muon} and SFT use the settings stated in their subsections. Global batches count sequences; the main GPT-2 budget therefore covers 24.576 billion token positions, and the other pretraining budgets cover 10.486 billion, before accounting for padding. Within each model, the main optimizer comparison uses a common training schedule.

\paragraph{Experimental organization.}
The main text compares optimization quality, quantization stability, memory savings, hyperparameter sensitivity, Muon compatibility, and fine-tuning. Appendix~\ref{sec:exp_app} gives training and parallelism settings and validation tables. Appendix~\ref{sec:additional} collects the supplementary experiments: matched-study measurements, repeated runs, Stacey and Lion~\citep{chen2024symbolic} comparisons, matrix-optimizer hybrids, parameter sweeps, and fine-tuning details. Appendix~\ref{sec:quantization-analysis} analyzes quantization and step-size sensitivity; Appendix~\ref{app:proofs} gives the convergence proofs.

\subsection{Comparison with Prior Methods}\label{sec:comparison}\label{app:trajectories}
On GPT-2-Medium, AdamW (\texttt{fp32}), PowerStep (\texttt{fp32}) and PowerStep (\texttt{int8}) share data order, validation batches, decay groups and a 9,000-update budget (4.424B token positions). Table~\ref{tab:matched-main} in Appendix~\ref{app:matched} reports validation means and sample standard deviations. The \texttt{fp32} and \texttt{int8} PowerStep configurations achieve slightly lower mean validation losses while reducing state storage from 2.642 to 1.321 and 0.341 GiB per device, respectively. In separate timing measurements using the same training shapes, peak allocation falls from 25.31 to 23.99 and 23.06 GiB, with throughput costs of 1.3\% and 3.6\%, respectively. Appendix~\ref{app:matched} gives the shared protocol and detailed results; Appendix~\ref{sec:exp_app} specifies the broader model comparison.

Figure~\ref{fig:small_exp} extends the comparison to six dense models with \texttt{fp32} states, 50,000 updates for GPT-2 and 20,000 for Qwen3. PowerStep's validation loss is 2.949 versus AdamW's 2.950 on GPT-2-Small and 0.006--0.028 higher on the other five models (Table~\ref{tab:valid_loss_small}). SignSGD is best on Small and Qwen3-0.6B; AdamS is best on Qwen3-4B. The cause of the high AdamS and SignSGD losses on 8B remains unresolved. Together, these results show competitive validation quality across model scales, with the matched study controlling data order and evaluation samples.

\begin{figure}[!p]
    \centering
    \subfloat[GPT-2-Small (124M)]{\includegraphics[width=0.44\linewidth]{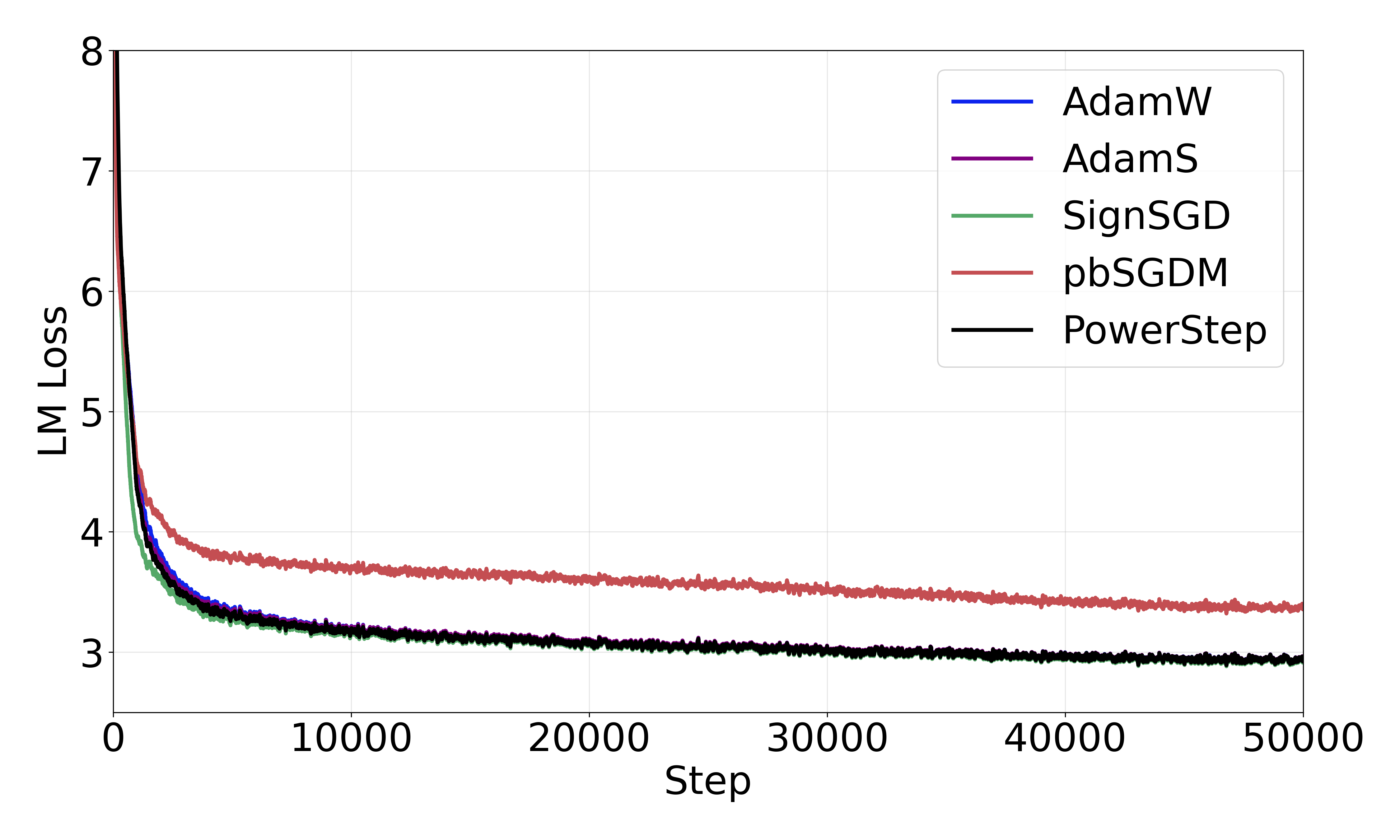}} 
    \subfloat[GPT-2-Medium (350M)]{\includegraphics[width=0.44\linewidth]{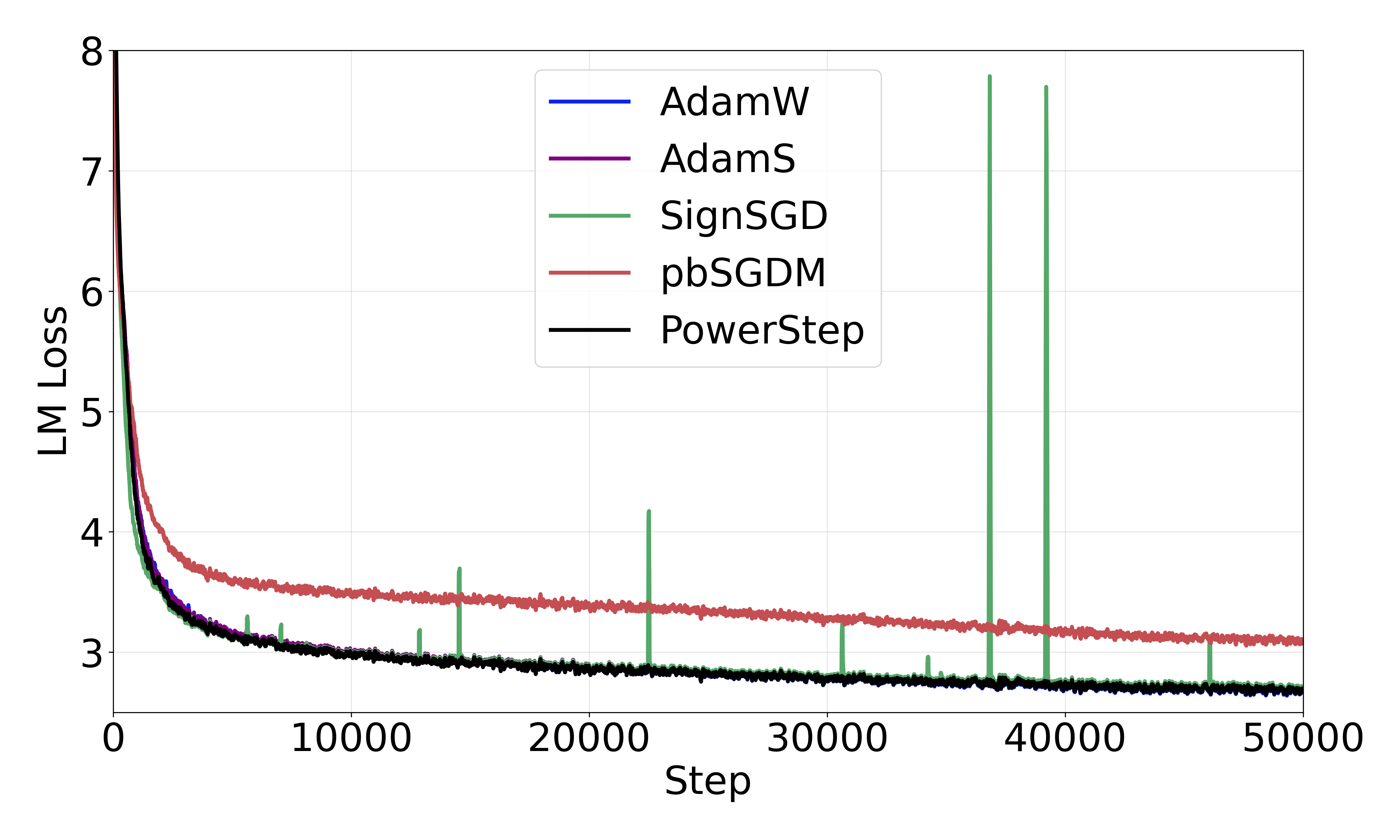}} 
    
    \subfloat[Qwen3-0.6B]{\includegraphics[width=0.44\linewidth]{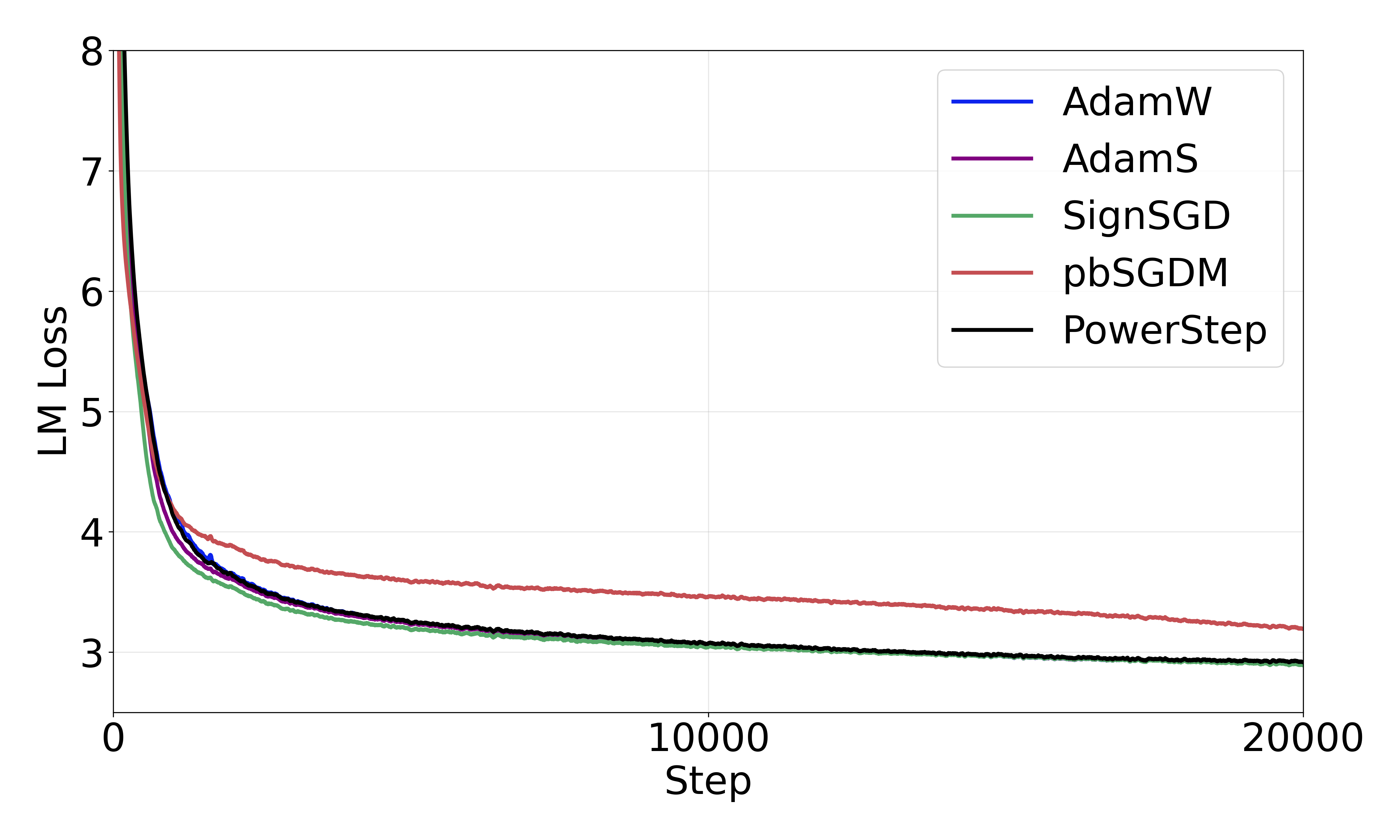}} 
    \subfloat[Qwen3-1.7B]{\includegraphics[width=0.44\linewidth]{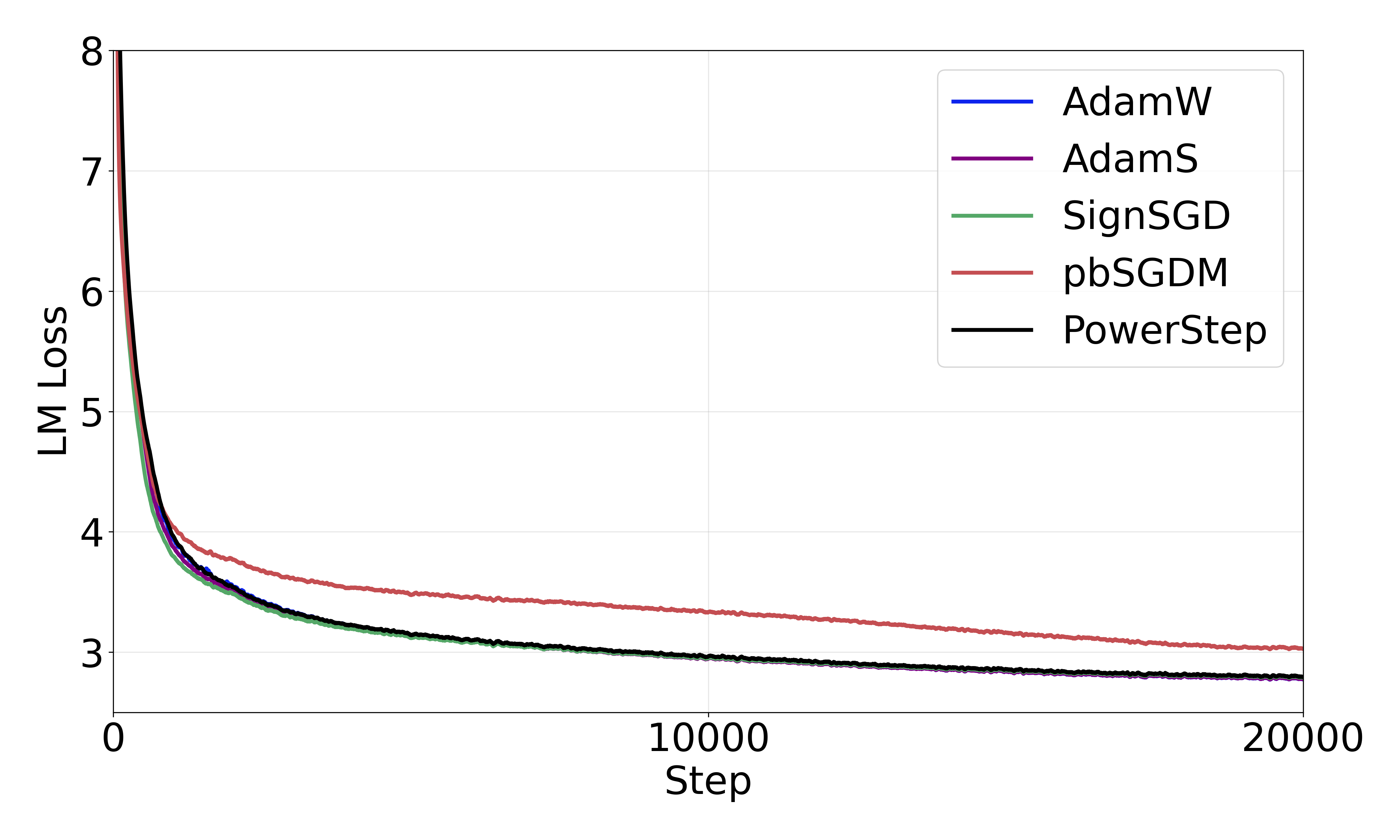}}
    
    \subfloat[Qwen3-4B]{\includegraphics[width=0.44\linewidth]{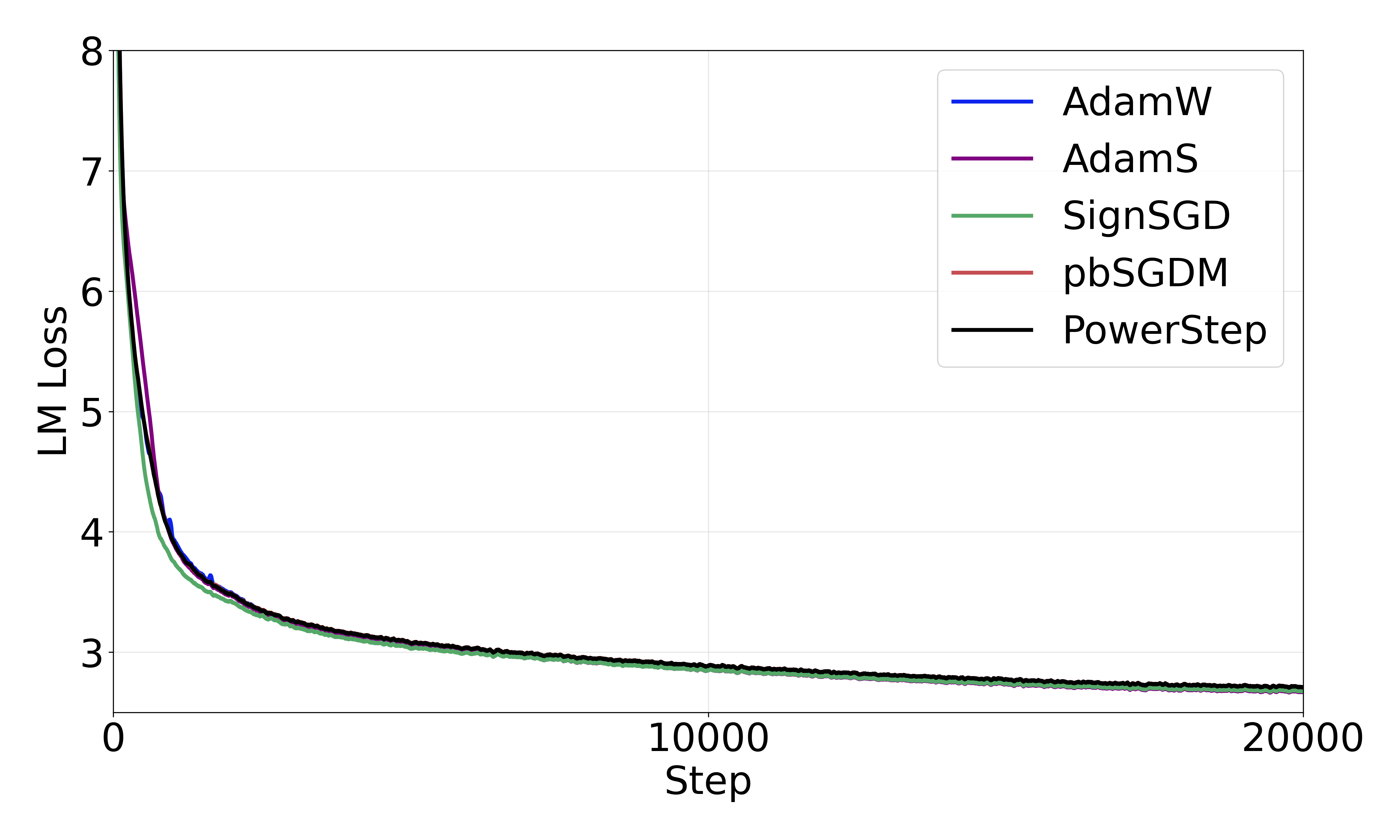}} 
    \subfloat[Qwen3-8B]{\includegraphics[width=0.44\linewidth]{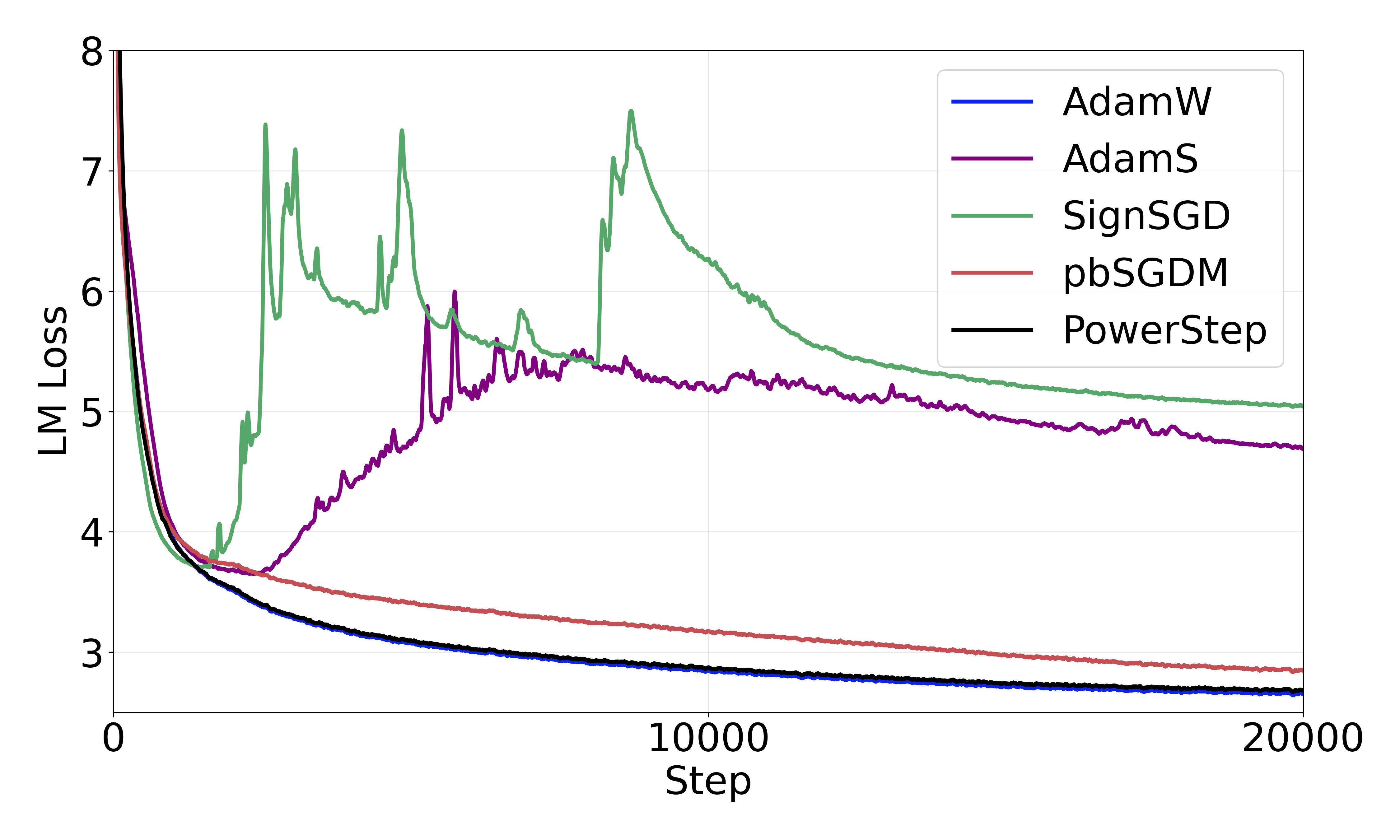}}
    \caption{Training losses with \texttt{fp32} states. Validation results are in Table~\ref{tab:valid_loss_small}.}
    \label{fig:small_exp}
\end{figure}

\begin{figure}[!p]
\centering
\subfloat[GPT-2-Small]{\includegraphics[width=.42\linewidth]{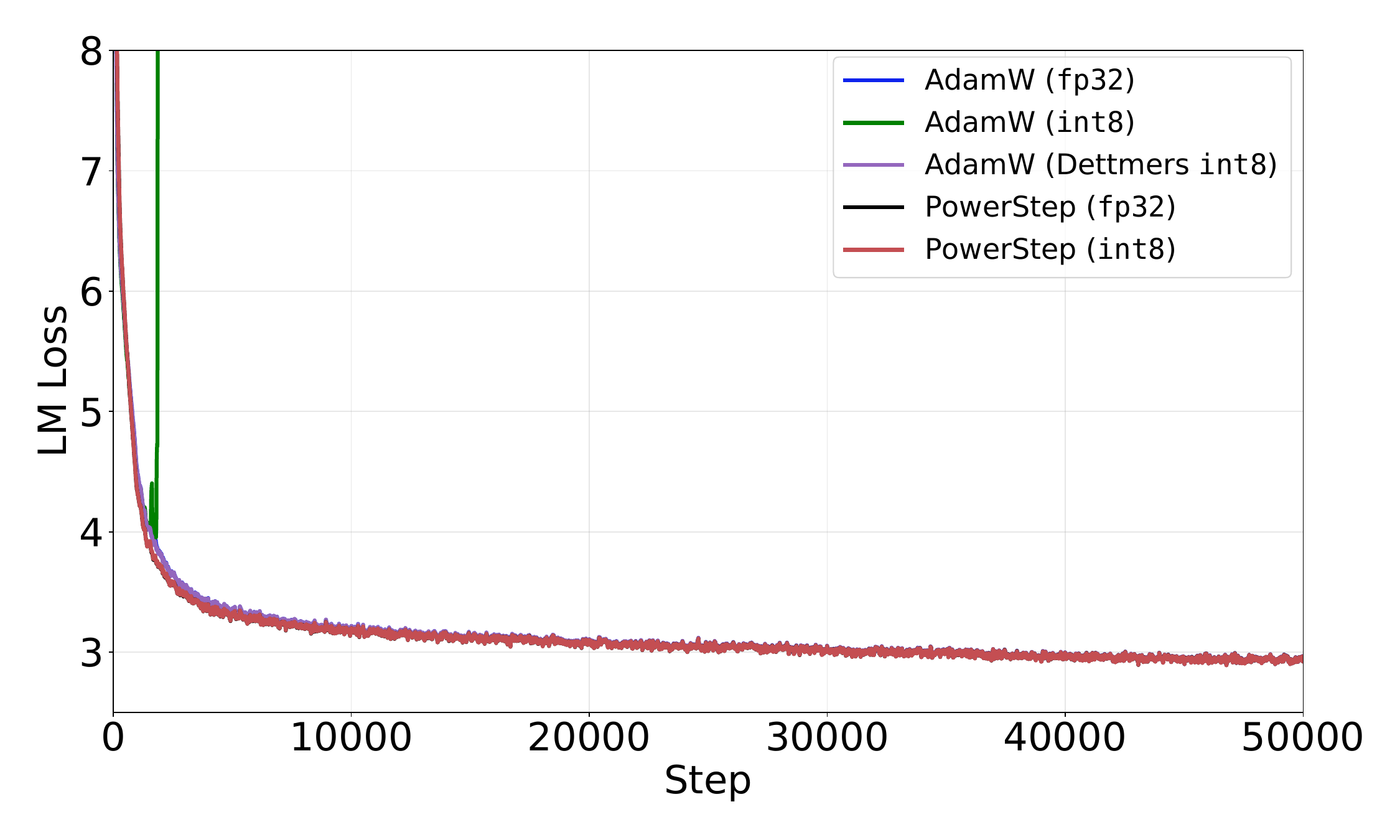}}\quad
\subfloat[GPT-2-Medium]{\includegraphics[width=.42\linewidth]{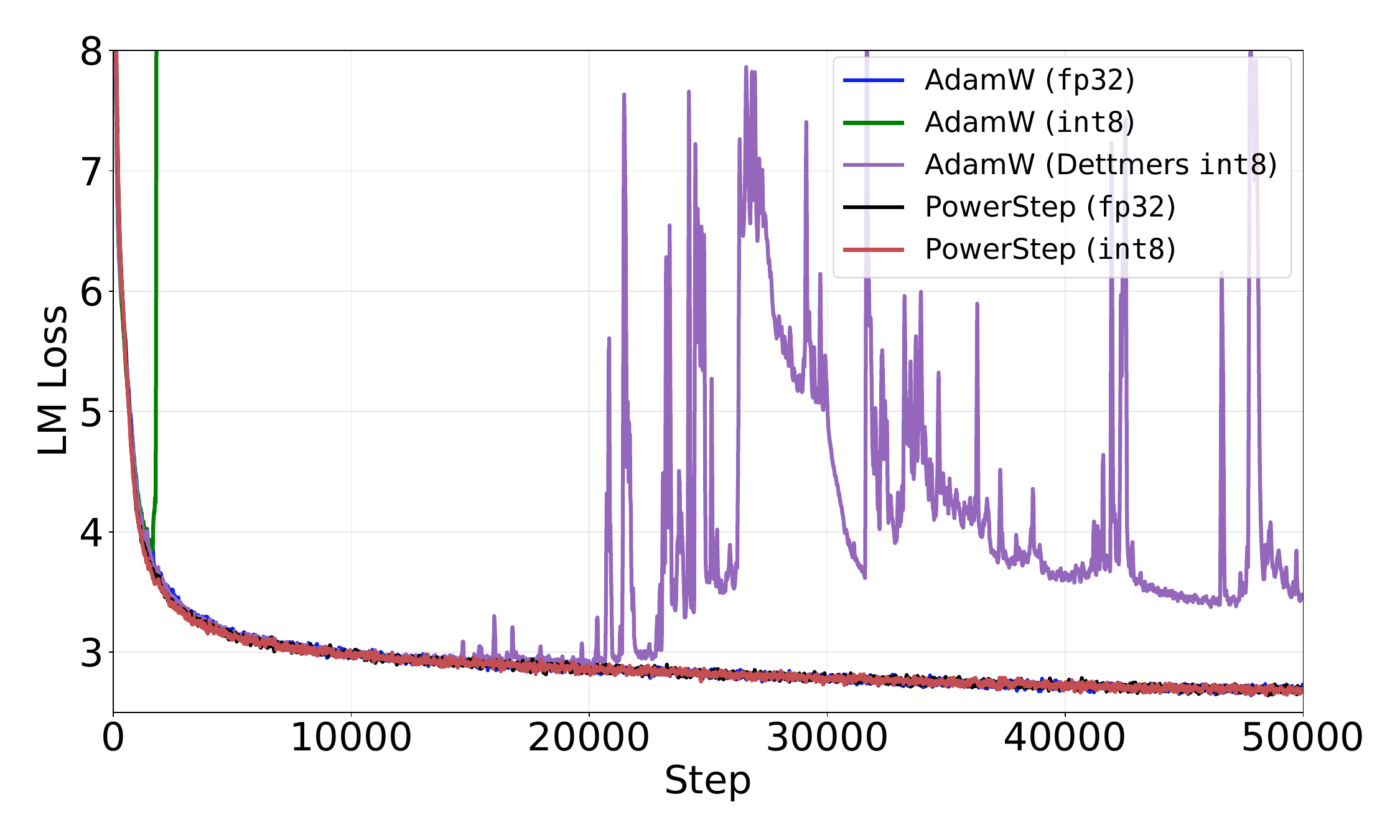}}
\caption{Training loss on GPT-2-Small and GPT-2-Medium with \texttt{fp32} and \texttt{int8} optimizer states. Curves show 50-step moving averages. Both Dettmers runs complete 50,000 updates; Medium exhibits repeated loss spikes. All quantized curves use 128-coordinate blocks. Experimental details are in Appendix~\ref{sec:quantization-analysis}.}
\label{fig:quant}
\end{figure}
\clearpage

\subsection{Quantization}\label{sec:quantization}
Figure~\ref{fig:quant} compares \texttt{fp32} and quantized states on GPT-2-Small and GPT-2-Medium. Both quantizers use 128-coordinate blocks. Uniform \texttt{int8} uses equally spaced levels for all states, including embeddings; Dettmers \texttt{int8}~\citep{dettmers2022bit} uses a nonuniform codebook and keeps embeddings and small tensors in \texttt{fp32}.

PowerStep trains stably with uniform \texttt{int8}, whereas AdamW diverges under the same scheme. Dettmers \texttt{int8} stabilizes AdamW on Small (validation loss 2.9524), but Medium exhibits repeated spikes and ends at 3.4768 after 50,000 updates. Underestimating AdamW's second moment can amplify its update through $1/(\sqrt{v}+\epsilon)$. Low-precision methods introduce additional stabilization: zero-excluding quantizers for 4-bit states~\citep{li2024memory}, stochastic rounding in Q-Adam-mini~\citep{han2025qadam}, and adaptive normalization, clipping, and moment resets in Stable-SPAM for 4-bit weights and activations~\citep{huang2025stablespam}.

PowerStep removes the second-moment denominator and attenuates small relative momentum errors away from zero, enabling simple uniform state quantization. Appendix~\ref{sec:quantization-analysis} gives the sensitivity analysis and Medium failure timeline; the traces do not isolate the trigger of its spikes.

\subsection{Large-Scale Experiments}\label{sec:large_scale}
Figure~\ref{fig:large_exp} compares PowerStep with uniformly quantized \texttt{int8} momentum states against AdamW with \texttt{fp32} states. PowerStep uses the 128-coordinate block-wise linear scheme in Section~\ref{sec:low_precision}, with one scale per block. Each model trains for 20,000 updates at global batch 256 and context length 2,048. Table~\ref{tab:memory} reports held-out loss and optimizer-state storage for the dense and sparse MoE models. Across these models, optimizer-state storage decreases by $\sim8\times$ (87.1\%, from $1-(1+4/128)/8$). For Qwen3-235B-A22B, it falls from 6,768.0 to 872.4 MB per device.

\begin{figure}[!htbp]
    \centering
    \subfloat[DeepSeek-V2-Lite (16B)]{\includegraphics[width=0.35\linewidth]{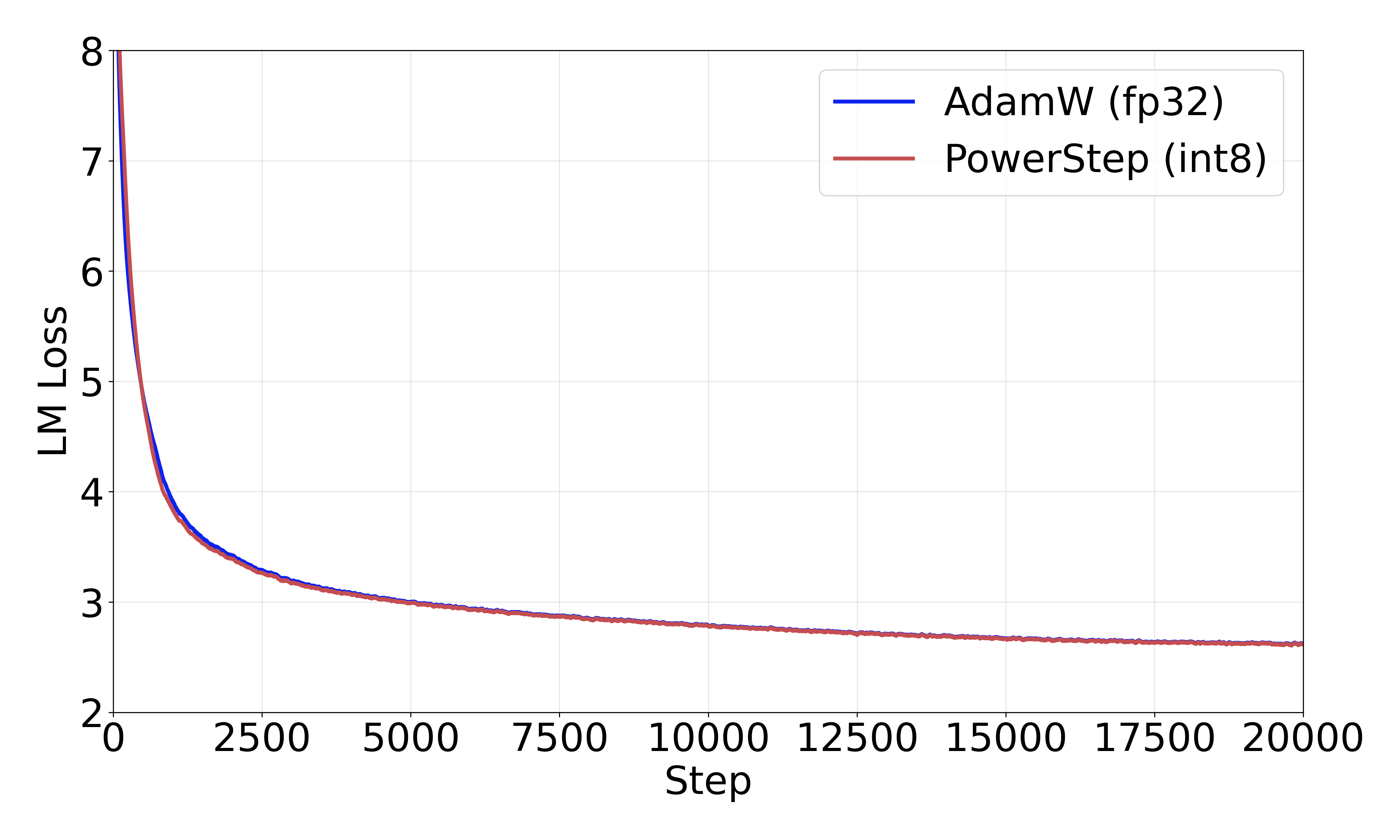}} \quad
    \subfloat[Qwen3-30B-A3B]{\includegraphics[width=0.35\linewidth]{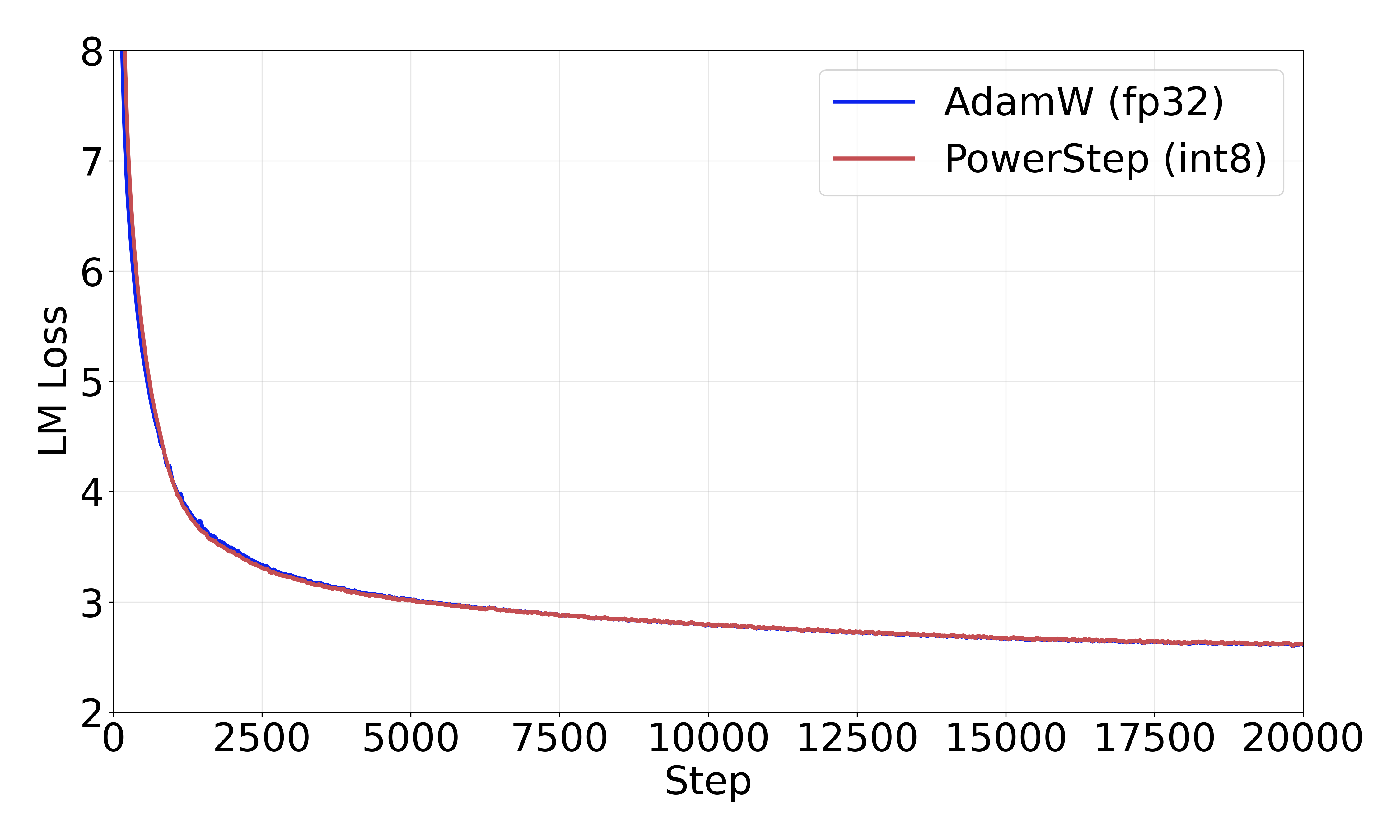}} \\
    \subfloat[Qwen3-32B]{\includegraphics[width=0.35\linewidth]{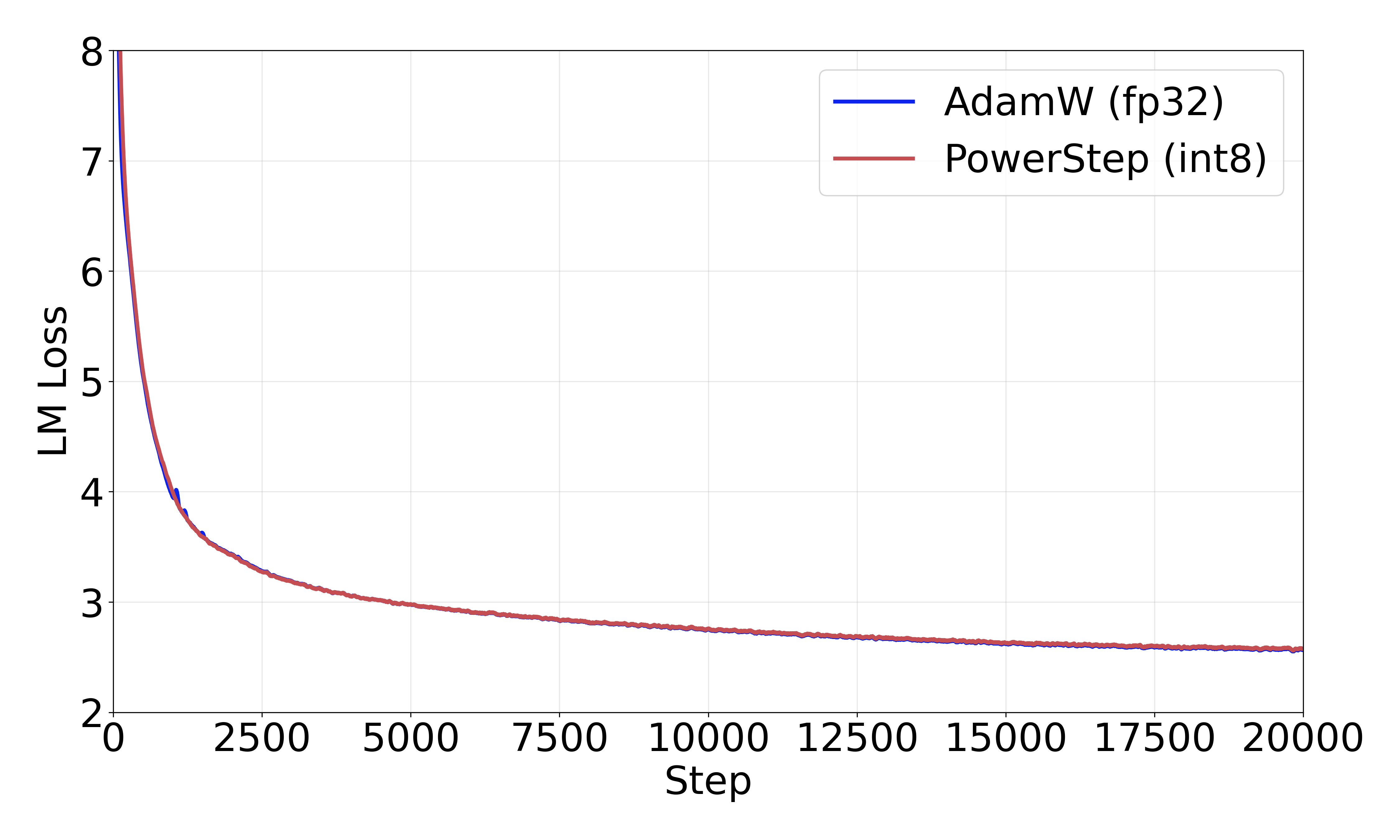}} \quad
    \subfloat[Qwen3-235B-A22B]{\includegraphics[width=0.35\linewidth]{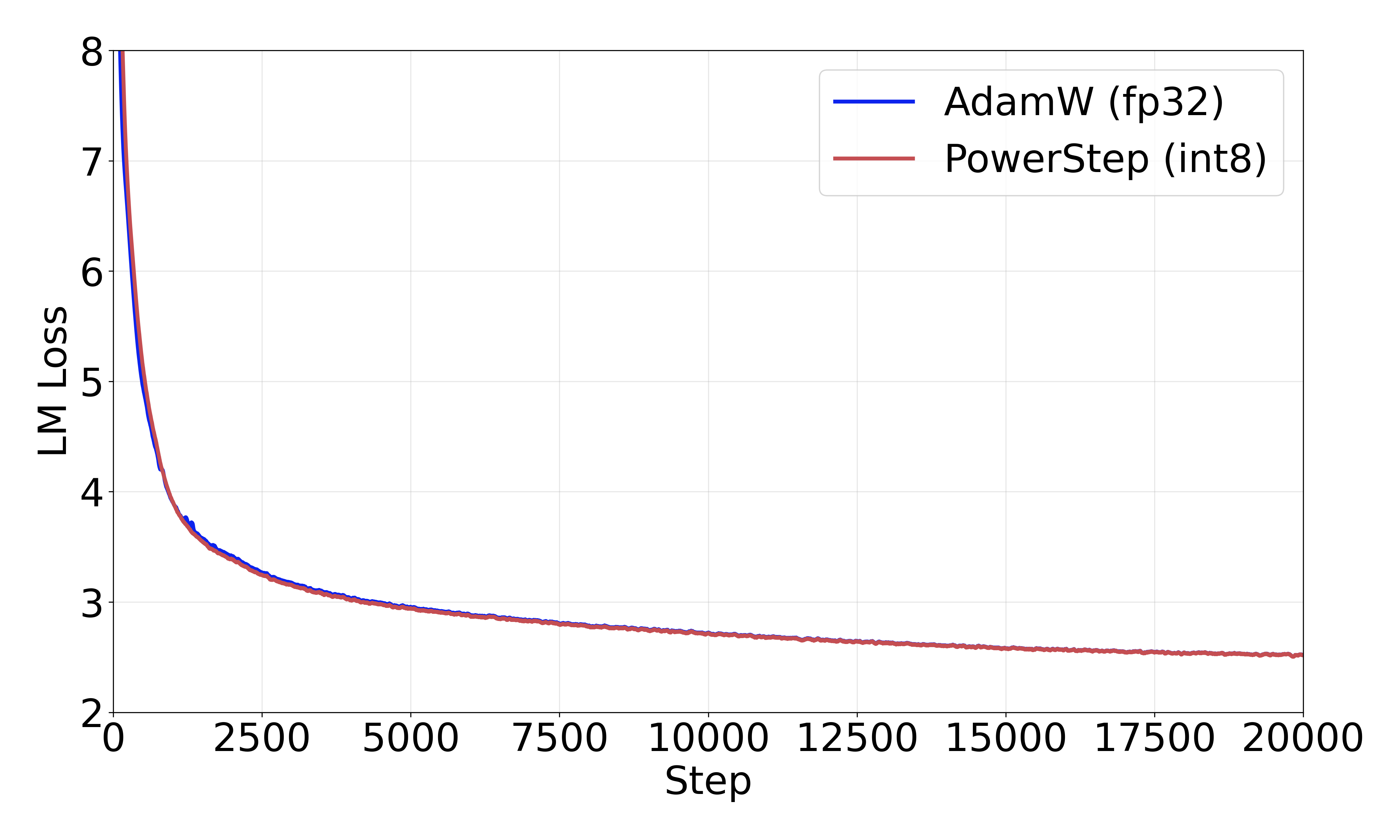}}
    \caption{Large-model training losses over 20,000 updates: PowerStep with uniform \texttt{int8} momentum quantization in 128-coordinate blocks and AdamW with \texttt{fp32} states.}
    \label{fig:large_exp}
\par\vspace{12pt}

\centering\small
\setlength{\tabcolsep}{5pt}
\begin{tabular}{lcccc}\toprule
&\multicolumn{2}{c}{Validation loss}&\multicolumn{2}{c}{Optimizer state (MB)}\\
Model & AdamW (\texttt{fp32}) & PowerStep (\texttt{int8}) & AdamW (\texttt{fp32}) & PowerStep (\texttt{int8})\\\midrule
DeepSeek-V2-Lite (16B)&2.623& \textbf{2.618} &429.0& \textbf{55.3} \\
Qwen3-30B-A3B& \textbf{2.620} &2.624&864.0& \textbf{111.4} \\
Qwen3-32B& \textbf{2.572} &2.582&976.5& \textbf{125.9} \\
Qwen3-235B-A22B&2.525& \textbf{2.524} &6768.0& \textbf{872.4} \\\bottomrule
\end{tabular}
\captionof{table}{Large-model results after 20,000 steps. State-memory columns report the recorded average optimizer-state storage per device, in MB.}
\label{tab:memory}

\end{figure}

\subsection{Hyperparameter Sensitivity}\label{app:lr-sensitivity}
\label{sec:lr}
We conduct ablations on the learning rate, power exponent $\beta$, and momentum coefficient $\gamma$. On GPT-2-Small, we vary $\eta_{\max}\in\{1,2,4,6,8,10\}\times10^{-4}$ with $\eta_{\min}=0.1\eta_{\max}$, keeping other settings fixed. Figure~\ref{fig:lr} shows faster early loss reduction at larger $\eta_{\max}$ for both AdamW and PowerStep. Figure~\ref{fig:hyper} varies $\beta\in\{0,0.1,0.2\}$ at $\gamma=0.9$, and $\gamma\in\{0.85,0.9,0.95\}$ at $\beta=0.1$, on GPT-2-Medium. 

Smaller $\beta$ compresses momentum magnitude ratios more strongly toward a sign update. Smaller values reduce loss faster initially; Appendix~\ref{app:sign-sensitivity} explains how near-sign updates can overshoot. In practice, we recommend $\beta$ near 0.1. Larger $\gamma$ retains the influence of past gradients for longer; our default is $(\beta,\gamma)=(0.1,0.9)$. In the separate exponent sweep (Table~\ref{tab:sweep}), observed validation endpoints vary by less than 0.04 on each GPT-2 model across $\beta\in\{0.05,0.10,0.15\}$.

\begin{figure}[!htb]
    \centering
    \subfloat[AdamW]{\includegraphics[width=0.35\linewidth]{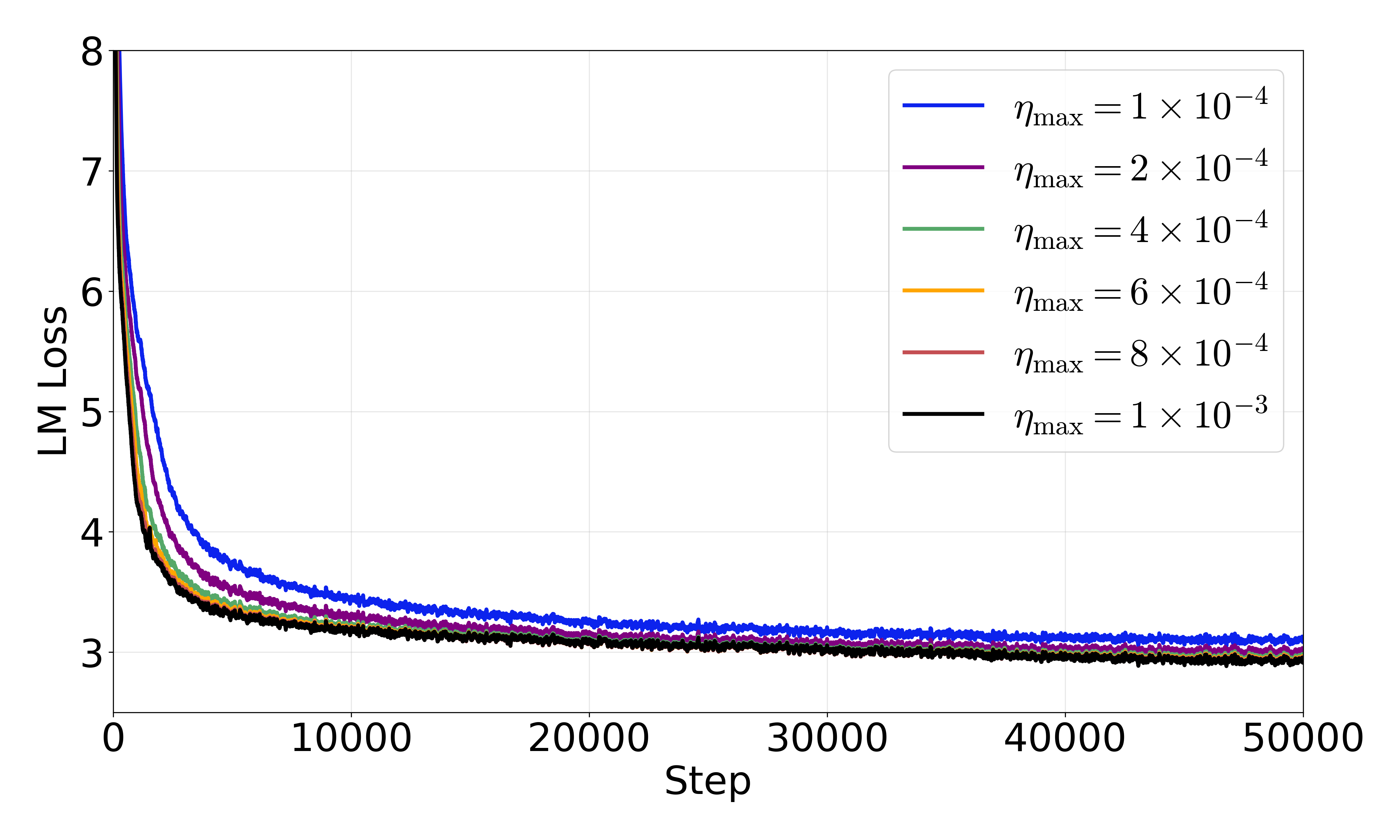}} \quad
    \subfloat[PowerStep]{\includegraphics[width=0.35\linewidth]{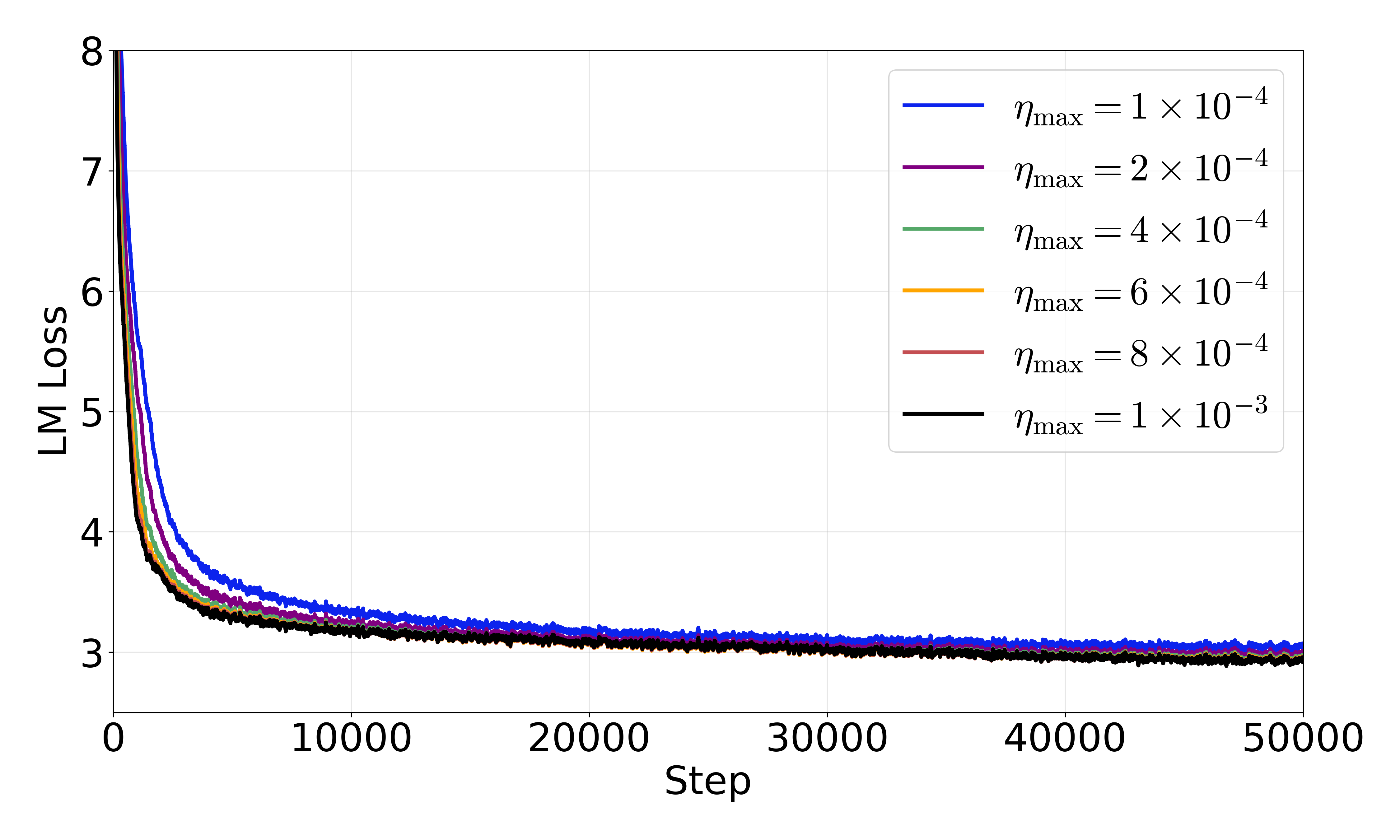}} 
    \caption{Sensitivity to $\eta_{\max}$ on GPT-2-Small (124M): AdamW (left) and PowerStep (right). Within each optimizer, we vary $\eta_{\max}$, with $\eta_{\min}=\eta_{\max}/10$; all other settings are fixed.}
    \label{fig:lr}

\end{figure}

\begin{figure}[!htb]
\centering
\subfloat[Power exponent $\beta$]{\includegraphics[width=.42\linewidth]{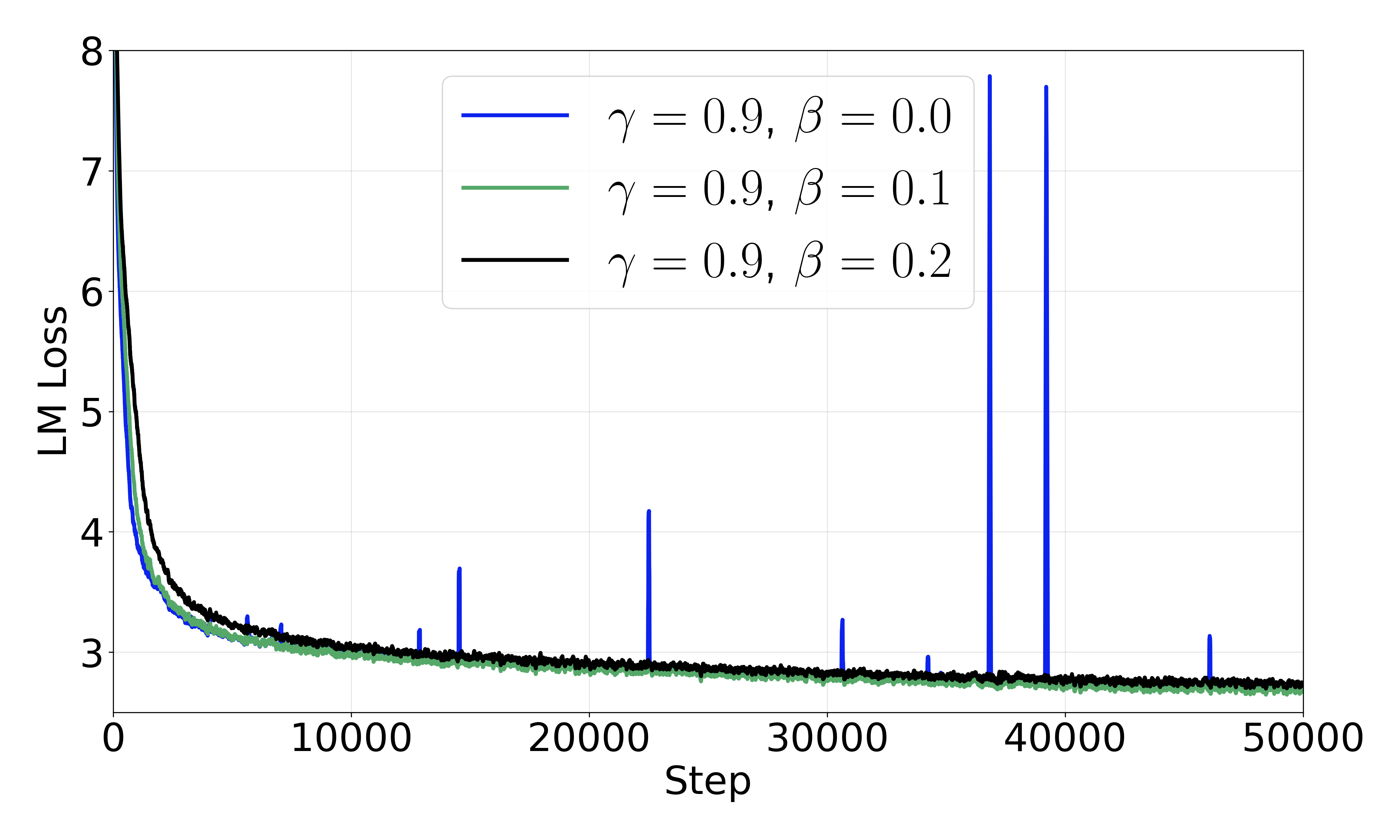}}
\subfloat[Momentum coefficient $\gamma$]{\includegraphics[width=.42\linewidth]{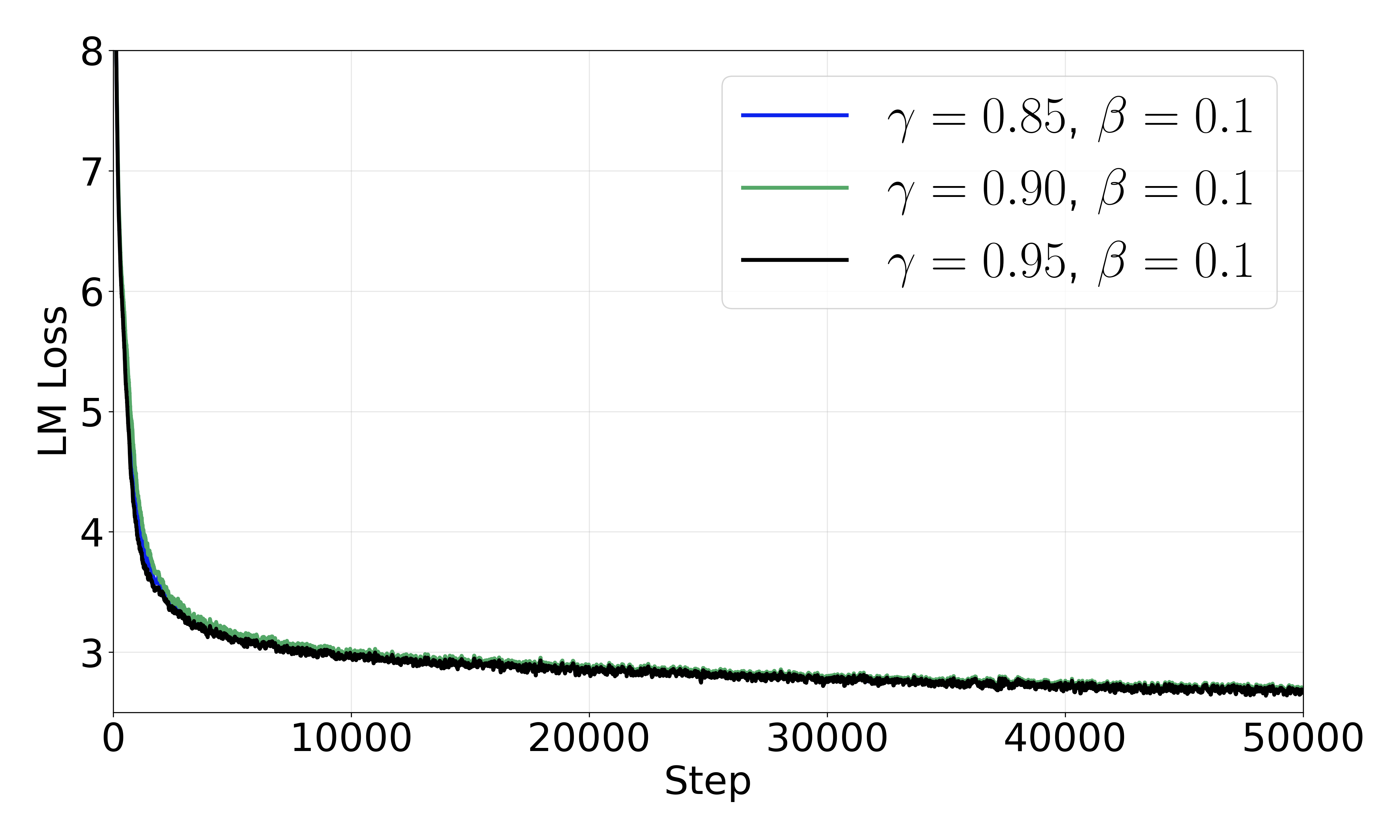}}
\caption{GPT-2-Medium sensitivity study. The $\beta=0$ curve has unresolved provenance and is shown for reference only: it differs from the SignSGD run in Figure~\ref{fig:small_exp}.}
\label{fig:hyper}
\end{figure}

\subsection{Compatibility with Muon}\label{sec:muon-main}
We test PowerStep as the auxiliary optimizer for Muon~\citep{jordan2024muon}: Muon updates the two-dimensional matrices inside Transformer blocks, while PowerStep updates embeddings, the tied output head and scalar/vector parameters. Figure~\ref{fig:muon} uses the same cosine schedule from $\eta_{\max}=6\times10^{-4}$ to $\eta_{\min}=6\times10^{-5}$ for both parameter groups in both hybrids, without warmup. On Small, Muon--AdamW reaches validation losses 2.9558 and 3.0353 across two runs at 50,000 updates (mean 2.9956), compared with 2.9438 for Muon--PowerStep. On Medium, the corresponding losses are 2.7065 for Muon--AdamW and 2.7159 for Muon--PowerStep.

The Small Muon--AdamW runs use fixed validation samples, while Muon--PowerStep uses the earlier evaluation protocol. The comparison demonstrates compatibility rather than isolating the auxiliary optimizer. Appendix~\ref{sec:hybrids} extends the comparison to SWAN~\citep{ma2025swan} and SinkGD~\citep{scetbon2025sinkgd}, which use stateless matrix updates. Their PowerStep hybrids reduce validation loss more slowly than the Muon hybrids in these runs. PowerStep accumulates gradients across steps through momentum; the stateless matrix updates transform only the current gradient. The appendix analyzes this trade-off.

\begin{figure}[!htb]
\centering
\subfloat[GPT-2-Small]{\includegraphics[width=.42\linewidth]{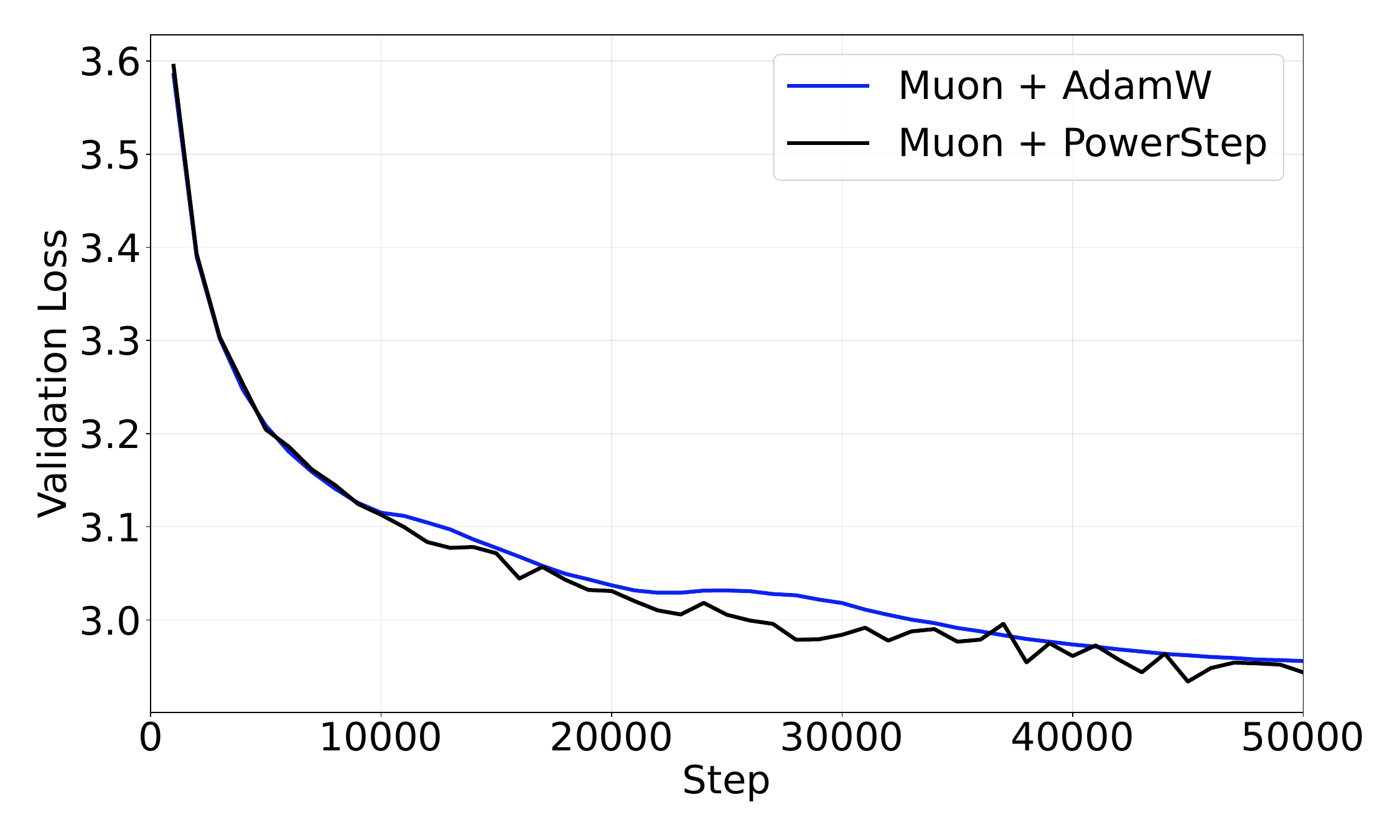}}\quad
\subfloat[GPT-2-Medium]{\includegraphics[width=.42\linewidth]{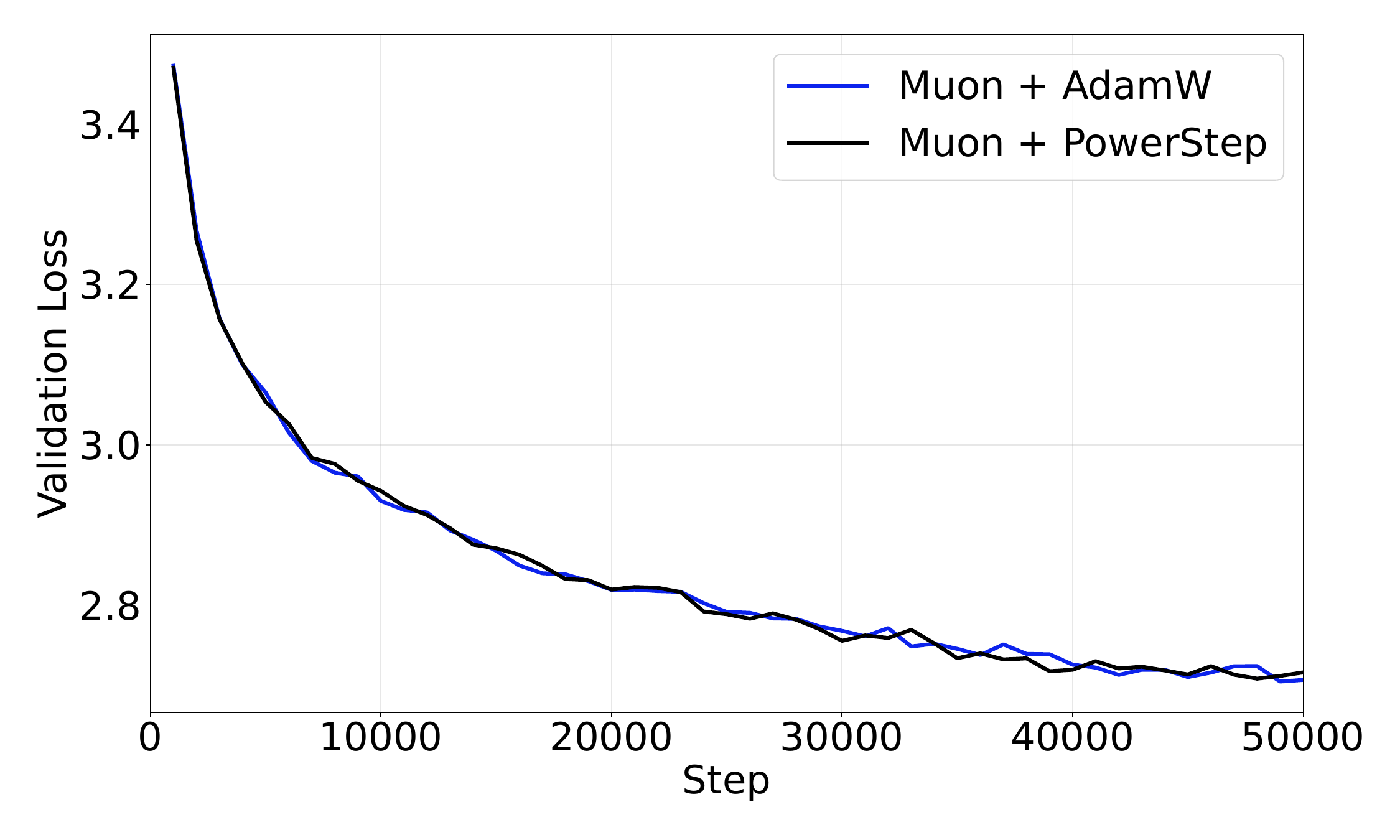}}
\caption{Muon hybrid validation loss on GPT-2-Small and GPT-2-Medium. The Small Muon--AdamW curve ends at 2.9558; the two-run mean is 2.9956. Validation protocols differ; the initial evaluation is omitted.}
\label{fig:muon}
\end{figure}

\subsection{Supervised Fine-Tuning}\label{sec:sft-main}
We fine-tune Qwen3-4B-Base on a filtered no-tool Nemotron-Math-v2 subset~\citep{nvidia2026nemotronmath}. Both runs complete 14,200 updates with global batch 128, context length 6,144, \texttt{bf16}, and linear decay from $\eta_{\max}=10^{-5}$ to $\eta_{\min}=10^{-6}$. Figure~\ref{fig:sft-training} shows lower training loss for PowerStep; the final 100-update means are 0.0932 versus 0.1053 for AdamW.

Both optimizers start from the same base checkpoint. Training responses contain boxed answers and span 1,000--4,000 tokens; weight decay is 0.1 and gradient clipping is 1.0. PowerStep improves Pass@1 by 1.25 percentage points on AMC23, 0.83 on AIME24 and 1.04 on AIME25 (Table~\ref{tab:sft}).

We sample 16 responses per question at temperature 0.6, top-$p=0.95$ and top-$k=40$, scoring the last boxed integer. Pass@1 averages per-question accuracy; Appendix~\ref{sec:sft} gives evaluation details.

\begin{figure}[!htb]
\centering\includegraphics[width=.50\linewidth]{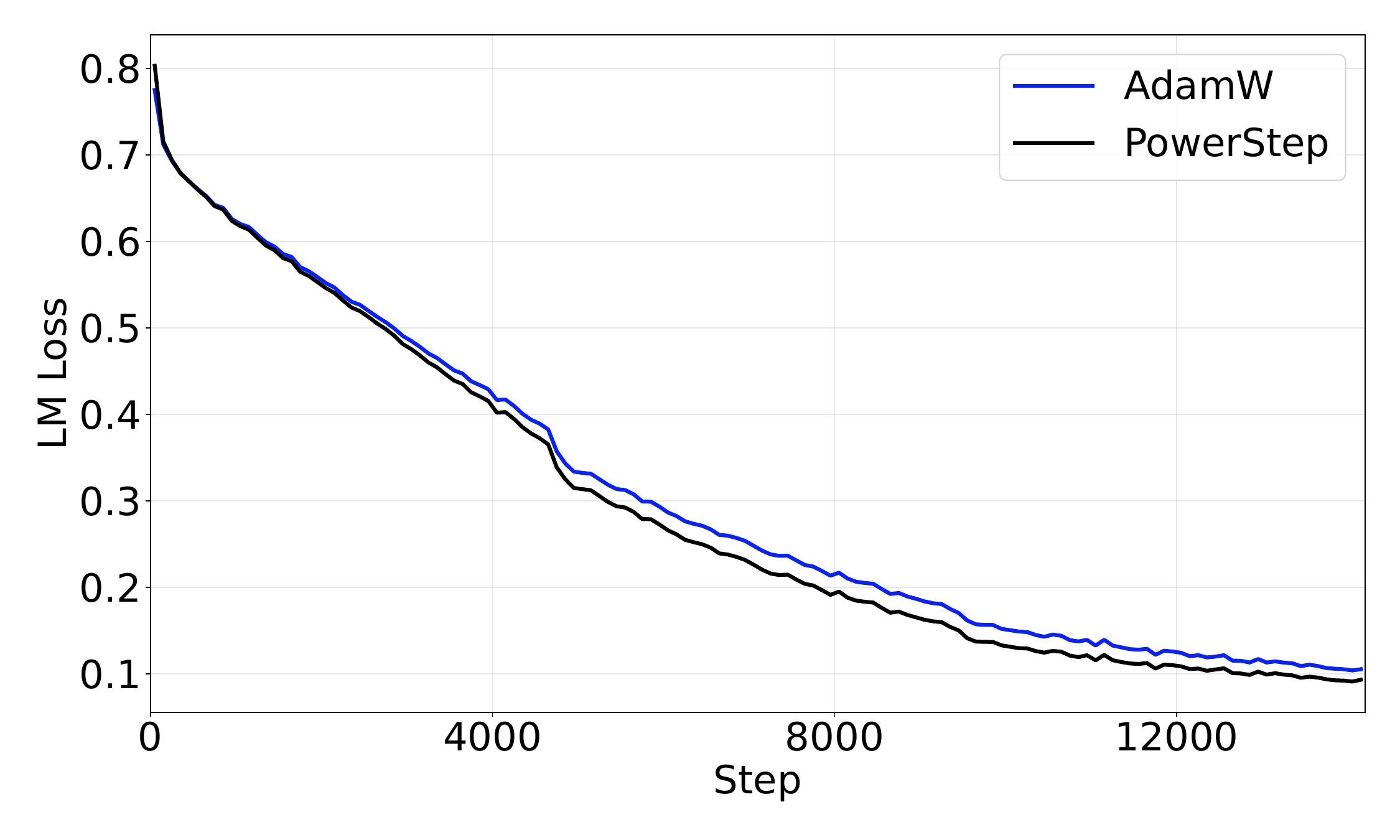}
\caption{Qwen3-4B SFT training loss, averaged over non-overlapping 100-update windows.}
\label{fig:sft-training}
\end{figure}

\begin{table}[H]
\centering\small

\begin{tabular}{ccc}\toprule
Benchmark & Optimizer & Pass@1 accuracy (\%)\\\midrule
\multirow{2}{*}{AMC23} & AdamW & 56.09\\
& PowerStep & \textbf{57.34} \\
\midrule
\multirow{2}{*}{AIME24} & AdamW & 18.54\\
& PowerStep & \textbf{19.38} \\
\midrule
\multirow{2}{*}{AIME25} & AdamW & 20.21\\
& PowerStep & \textbf{21.25} \\
\bottomrule
\end{tabular}
\caption{Final SFT checkpoint evaluation. Pass@1 is estimated from 16 responses per question: 40 AMC23 questions and 30 per AIME set.}
\label{tab:sft}
\end{table}

\section{Discussion and Conclusion}\label{sec:conclusion}
Compared with AdamW, PowerStep achieves competitive validation quality with one momentum buffer, halving \texttt{fp32} optimizer-state memory. In our matched measurements, it also reduces peak allocation at a modest throughput cost. Uniform \texttt{int8} state quantization further reduces optimizer-state storage to approximately $1/8$ of that required by \texttt{fp32} AdamW.

Future work includes extending PowerStep to lower-bit state storage and combining it with subspace-based compression. Analyzing quantized updates and adaptive choices of $\beta$ could further connect numerical robustness with convergence, guiding single-state optimizer design for larger models.

\FloatBarrier
\bibliographystyle{plainnat}
{\interlinepenalty=0\bibliography{main}}
\newpage
\appendix
\section{Experimental Details}
\label{sec:exp_app}
\subsection{Hyperparameters}

Table~\ref{tab:lr} reports the main pretraining learning rates. Weight decay is $0.1$, global gradient norm clipping is $1.0$, and the MoE auxiliary load-balancing coefficient is $10^{-3}$. Warmup takes 2,000 steps, followed by cosine decay to the stated minimum at step 50,000 for GPT-2 or step 20,000 for all other models. Global batch size counts sequences; nominal token totals are batch size times context length times steps, without correcting for padding.

The main comparison uses a common schedule within each model. Section~\ref{sec:lr} varies the learning rate for AdamW and PowerStep on GPT-2-Small; Table~\ref{tab:sweep} separately varies AdamW's $\beta_2$ and PowerStep's exponent. In particular, pbSGDM uses the same $\gamma$ and $\beta$ as PowerStep in the main comparison.

\begin{table}[!htbp]
    \centering
    \small        
    
    \begin{tabular}{l c c c c}
        \toprule
        {Model}  & Batch size & Context length & $\eta_{\max}$ & $\eta_{\min}$ \\
        \midrule
        GPT-2-Small (124M) & 480 & 1024 & $\text{6}\times \text{10}^{-\text{4}}$ & $\text{6}\times \text{10}^{-\text{5}}$ \\
        GPT-2-Medium (350M) & 480 & 1024& $\text{6}\times \text{10}^{-\text{4}}$ & $\text{6}\times \text{10}^{-\text{5}}$ \\
        Qwen3-0.6B & 256 & 2048 & $\text{5}\times \text{10}^{-\text{4}}$ & $\text{5}\times \text{10}^{-\text{5}}$ \\
        Qwen3-1.7B & 256 & 2048& $\text{5}\times \text{10}^{-\text{4}}$ & $\text{5}\times \text{10}^{-\text{5}}$ \\
        Qwen3-4B & 256 & 2048& $\text{5}\times \text{10}^{-\text{4}}$ & $\text{5}\times \text{10}^{-\text{5}}$ \\
        Qwen3-8B & 256 & 2048& $\text{3}\times \text{10}^{-\text{4}}$ & $\text{3}\times \text{10}^{-\text{5}}$ \\
        DeepSeek-V2-Lite (16B) & 256 & 2048& $\text{2}\times \text{10}^{-\text{4}}$ & $\text{2}\times \text{10}^{-\text{5}}$ \\
        Qwen3-30B-A3B & 256 & 2048   & $\text{2}\times \text{10}^{-\text{4}}$ & $\text{2}\times \text{10}^{-\text{5}}$ \\
        Qwen3-32B   & 256 & 2048    & $\text{2}\times \text{10}^{-\text{4}}$ & $\text{2}\times \text{10}^{-\text{5}}$ \\
        Qwen3-235B-A22B & 256 & 2048 & $\text{2}\times \text{10}^{-\text{4}}$ & $\text{2}\times \text{10}^{-\text{5}}$ \\
        \bottomrule
    \end{tabular}
\caption{Training hyperparameters}
    \label{tab:lr}
    
\end{table}

\subsection{Parallelism Configuration}

Table~\ref{tab:para} lists the parallelism configuration for each model. We denote data parallelism by DP, tensor parallelism by TP, pipeline parallelism by PP and expert parallelism by EP. The total number of NPUs satisfies $\text{NPUs} = \text{DP} \times \text{TP} \times \text{PP}$. For MoE architectures, expert parallelism partitions the data-parallel ranks, with $\text{EP} \leq \text{DP}$; it is not an additional factor in the device count.

\begin{table}[!htbp]
    \centering
    \small        
    
    \begin{tabular}{l c c c c c}
        \toprule
        {Model}  & NPUs & DP & TP & PP & EP \\
        \midrule
        GPT-2-Small (124M)   & 8   & 8  & 1 & 1 & N/A \\
        GPT-2-Medium (350M)  & 8   & 8  & 1 & 1 & N/A \\
        Qwen3-0.6B           & 8   & 8  & 1 & 1 & N/A \\
        Qwen3-1.7B           & 8   & 8  & 1 & 1 & N/A \\
        Qwen3-4B             & 8   & 8  & 1 & 1 & N/A \\
        Qwen3-8B             & 32  & 16 & 1 & 2 & N/A \\
        DeepSeek-V2-Lite (16B) & 256 & 256 & 1 & 1 & 8 \\
        Qwen3-30B-A3B        & 256  & 64 & 1 & 4 & 4 \\
        Qwen3-32B            & 256 & 16 & 8 & 2 & N/A \\
        Qwen3-235B-A22B      & 256 & 64  & 1 & 4 & 8 \\
        \bottomrule
    \end{tabular}
\caption{Parallelism configuration across models}
    \label{tab:para}
    
\end{table}

\FloatBarrier
\subsection{Validation Results}
\label{sec:val_loss}

Validation losses for dense and large-scale models are reported in Tables~\ref{tab:valid_loss_small} and~\ref{tab:memory}, respectively. The impact of optimizer state quantization on validation performance is examined separately in Table~\ref{tab:quant_valid}.

The matched study is reported separately in Appendix~\ref{app:matched}.

\begin{table}[!htbp]
    \centering
    \small        
    
    \begin{tabular}{l c c c c c}
        \toprule
        {Model}  & AdamW & AdamS & SignSGD & pbSGDM & PowerStep \\
        \midrule
        GPT-2-Small (124M) & 2.950 & 2.951 & \textbf{2.942} & 3.377 & 2.949 \\
        GPT-2-Medium (350M) & \textbf{2.711} & 2.725 & 2.736 & 3.110 & 2.717 \\
        Qwen3-0.6B & 2.912 & 2.912 & \textbf{2.897} & 3.198 & 2.923 \\
        Qwen3-1.7B & \textbf{2.780} & {2.782} & 2.794 & 3.034 & 2.799 \\
        Qwen3-4B & 2.690 & \textbf{2.672} & 2.680 & 2.712 & 2.712 \\
        Qwen3-8B & \textbf{2.660} & 4.689 & 5.045 & 2.853 & 2.688 \\    
        \bottomrule
    \end{tabular}
\caption{Validation losses on dense models, all states \texttt{fp32}. The lowest validation loss in each row is bold.}
    \label{tab:valid_loss_small}    
\end{table}

\begin{table}[!htbp]
    \centering
    \small        
    
    \begin{tabular}{l c c c}
        \toprule
        {Model}  & {AdamW (\texttt{fp32})} & {PowerStep (\texttt{fp32})} & {PowerStep (\texttt{int8})} \\
        \midrule
        GPT-2-Small (124M)   & 2.950 & 2.949 & \textbf{2.948} \\
        GPT-2-Medium (350M)  & \textbf{2.711} & 2.717 & 2.716 \\
        \bottomrule
    \end{tabular}
\caption{Validation losses under state quantization. Uniform all-layer \texttt{int8} AdamW diverged and is omitted. Dettmers \texttt{int8} reaches 2.9524 on Small and 3.4768 on Medium at 50,000 updates (Section~\ref{sec:quantization}).}
    \label{tab:quant_valid}    
\end{table}

\FloatBarrier
\section{Supplementary Experiments}\label{sec:additional}
This appendix collects the matched-budget and multiple-seed studies, optimizer comparisons, hyperparameter sweeps, and fine-tuning details. Each subsection specifies the corresponding training and evaluation protocol.
\subsection{Matched Comparison at 9,000 Updates}\label{app:matched}\label{sec:matched-validation}\label{sec:matched}
We compare AdamW (\texttt{fp32}), PowerStep (\texttt{fp32}), and PowerStep (\texttt{int8}) on GPT-2-Medium, pairing initialization, training data order, validation samples, and parameter decay groups across configurations. Each configuration completes 9,000 updates for seeds 1337, 1338, and 1339, covering 4.424 billion token positions per run.

\paragraph{Training protocol.}
The model has 24 layers, 16 attention heads, hidden size 1,024, vocabulary size 50,304, and no bias or dropout. All runs start from scratch on tokenized OpenWebText, using \texttt{fp32} parameters with \texttt{bf16} autocast and context length 1,024. Eight Ascend 910 devices each process microbatches of 12 sequences and accumulate five gradients, giving global batch 480. Gradients are synchronized and clipped to norm 1.0. Weight decay 0.1 applies to matrix parameters, with no decay on vectors or scalars. AdamW uses $(\beta_1,\beta_2)=(0.9,0.95)$ and $\epsilon=10^{-8}$; PowerStep uses $(\gamma,\beta)=(0.9,0.1)$. All configurations share 900 warmup updates to $\eta_{\max}=6\times10^{-4}$ and cosine decay to $\eta_{\min}=6\times10^{-5}$.

The 9,000-update horizon was selected before inspecting losses from a 24-hour allocation on four nodes. The calculation used the slowest measured step time, a 15\% timing allowance, and a 30-minute reserve, rounding down to 1,000 updates with a cap of 10,000.

\paragraph{Paired sampling and evaluation.}
For initialization seed $s$, rank $r$ samples training batches with a separate generator seeded by $s+100000+r$. Evaluation uses seed $918273+r$, reset at each checkpoint, and does not advance the training generator. Each rank evaluates 25 batches, totaling 2,400 sequences per checkpoint. Evaluation and checkpointing occur every 1,000 updates, including update 9,000.

\paragraph{Results.}
Table~\ref{tab:matched-main} summarizes validation quality and memory; Table~\ref{tab:all-seeds} lists the individual endpoints. Both PowerStep variants achieve slightly lower mean validation loss than AdamW, while reducing optimizer-state storage and measured peak allocation. All completed runs contribute to the reported means and sample standard deviations.

\begin{table}[!htb]
\centering\small

\begin{tabular}{lccc}\toprule
Configuration & Validation loss (mean $\pm$ SD) & State & Peak \\\midrule
AdamW (\texttt{fp32}) & 2.9563$\pm$0.0026 & 2.642 & 25.31\\
PowerStep (\texttt{fp32}) & 2.9487$\pm$0.0056 & 1.321 & 23.99\\
PowerStep (\texttt{int8}) & \textbf{2.9484$\pm$0.0047} & \textbf{0.341} & \textbf{23.06}\\\bottomrule
\end{tabular}
\caption{Matched three-seed study on GPT-2-Medium at 9,000 updates. Memory is GiB per device; peak allocation is measured separately.}
\label{tab:matched-main}
\end{table}

\begin{table}[ht]
\centering\small

\begin{tabular}{lccccc}\toprule
Optimizer & 1337 & 1338 & 1339 & $n$ & Mean $\pm$ SD\\\midrule
AdamW (\texttt{fp32}) & 2.9593 & 2.9549 & 2.9547 & 3 & 2.9563$\pm$0.0026 \\
PowerStep (\texttt{fp32}) & 2.9535 & 2.9500 & \textbf{2.9426}  & 3 & 2.9487$\pm$0.0056 \\
PowerStep (\texttt{int8}) & \textbf{2.9530}  & \textbf{2.9436}  & 2.9487 & 3 & \textbf{2.9484$\pm$0.0047}  \\
\bottomrule\end{tabular}
\caption{All three completed 9,000-step validation endpoints per method. The final column reports the mean and sample SD across the three seeds.}
\label{tab:all-seeds}
\end{table}

\FloatBarrier
\subsection{Multiple Seeds at 50,000 Updates}\label{sec:long-seeds}
This study evaluates variation over the full GPT-2 training schedule, complementing the paired 9,000-update comparison in Appendix~\ref{app:matched}. AdamW and PowerStep each have five completed runs on GPT-2-Small and GPT-2-Medium, using the main-study $\eta_{\max}=6\times10^{-4}$ and 50,000 updates.

Table~\ref{tab:seeds} reports every endpoint together with the mean, sample standard deviation, and standard error of the mean (SEM). On Small, both optimizers finish near 2.95. On Medium, PowerStep endpoints range from 2.7143 to 2.7342, whereas two AdamW runs end above 6.75 and raise its aggregate mean. These runs remain included in the statistics.

\begin{table}[ht]
\centering\small

\begin{tabular}{llrrrrr}\toprule
Model & Optimizer & Run 1 & Run 2 & Run 3 & Run 4 & Run 5\\\midrule
Small & AdamW & 2.9525 & 2.9628 & 2.9528 & 2.9472 & 2.9503\\
Small & PowerStep & 2.9632 & 2.9625 & 2.9465 & 2.9550 & 2.9486\\
Medium & AdamW & 2.7494 & 6.7695 & 2.7180 & 6.7551 & 2.7105\\
Medium & PowerStep & 2.7143 & 2.7189 & 2.7342 & 2.7180 & 2.7170\\\bottomrule
\end{tabular}

\medskip
\begin{tabular}{ccccc}\toprule
Model & Optimizer & Mean $\pm$ SEM & Sample SD & High-loss runs\\\midrule
\multirow{2}{*}{Small} & AdamW & \textbf{2.9531$\pm$0.0026}  & 0.0059 & 0/5\\
& PowerStep & 2.9552$\pm$0.0034 & 0.0077 & 0/5\\
\multirow{2}{*}{Medium} & AdamW & 4.3405$\pm$0.9887 & 2.2108 & 2/5\\
& PowerStep & \textbf{2.7205$\pm$0.0035}  & 0.0079 & 0/5\\\bottomrule
\end{tabular}
\caption{Individual endpoints and five-run aggregates. Run numbers index records within each row; they do not indicate paired initialization across optimizers. Uncertainty is SEM ($s/\sqrt5$), not SD. High-loss counts use the descriptive threshold final validation loss $>4$; Medium AdamW includes both high-loss endpoints.}
\label{tab:seeds}
\end{table}

Four runs in each group have explicit seed records; the fifth is matched to an archived run. Two near-identical archived Medium candidates leave a provenance ambiguity that does not change the rounded mean or SEM. Unlike Appendix~\ref{app:matched}, these records do not establish paired initialization and data order across optimizers. They describe variation under the recorded recipe; the two high-loss AdamW runs do not by themselves isolate its cause.

\FloatBarrier
\subsection{Comparison to Stacey}
\label{sec:stacey}

PowerStep can be viewed as a simplified, memory-efficient variant of Stacey-$(p,2)$~\citep{luo2025stacey} that removes the primal--dual auxiliary variables and the $\epsilon$-stabilization term and employs heavy-ball momentum. To assess whether these simplifications incur any performance cost, we conduct a direct comparison between the two optimizers on GPT-2-Small (124M) and GPT-2-Medium (350M).
For Stacey-$(p,2)$, we adopt the hyperparameters from~\citep{luo2025stacey}: $\alpha=0.1$, $\beta_1=0.9$, $\beta_2=0.99$, $\tau=0.001$, and $\epsilon=1\times 10^{-8}$. We set $p=11$, corresponding to $\beta = 1/(p-1) = 0.1$, which matches PowerStep's power exponent. For learning rates, we evaluate $\eta_{\max} \in \{6 \times 10^{-4}, 1 \times 10^{-3}\}$ with $\eta_{\min} = 0.1 \cdot \eta_{\max}$ across both optimizers. Other pretraining settings follow Appendix~\ref{sec:exp_app}.
Figure~\ref{fig:stacey} reports the training loss trajectories. The plotted trajectories show faster early loss reduction for PowerStep and close final training losses.

\begin{figure}[!htbp]
    \centering    
    \subfloat[GPT-2-Small (124M)]{\includegraphics[width=0.45\linewidth]{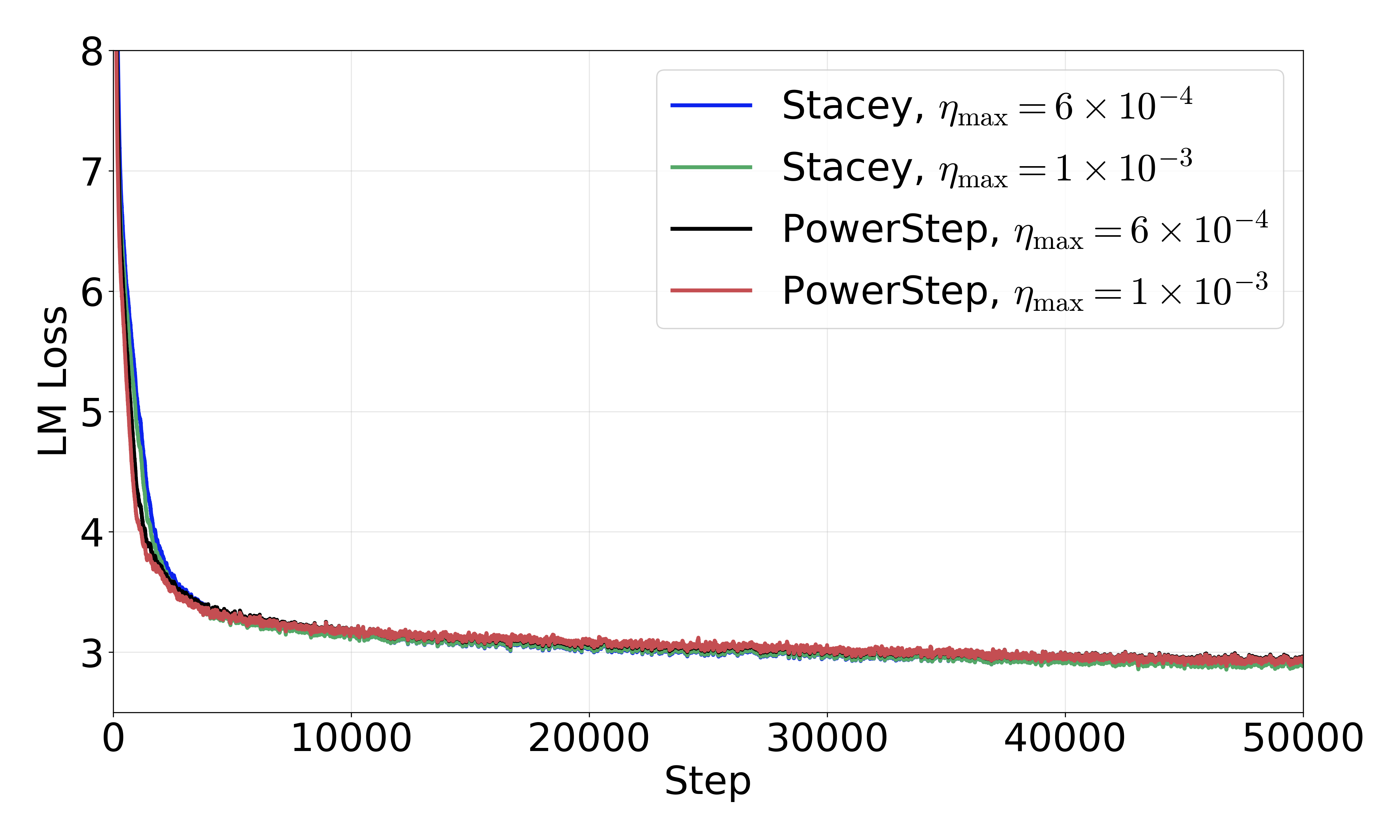}} 
    \subfloat[GPT-2-Medium (350M)]{\includegraphics[width=0.45\linewidth]{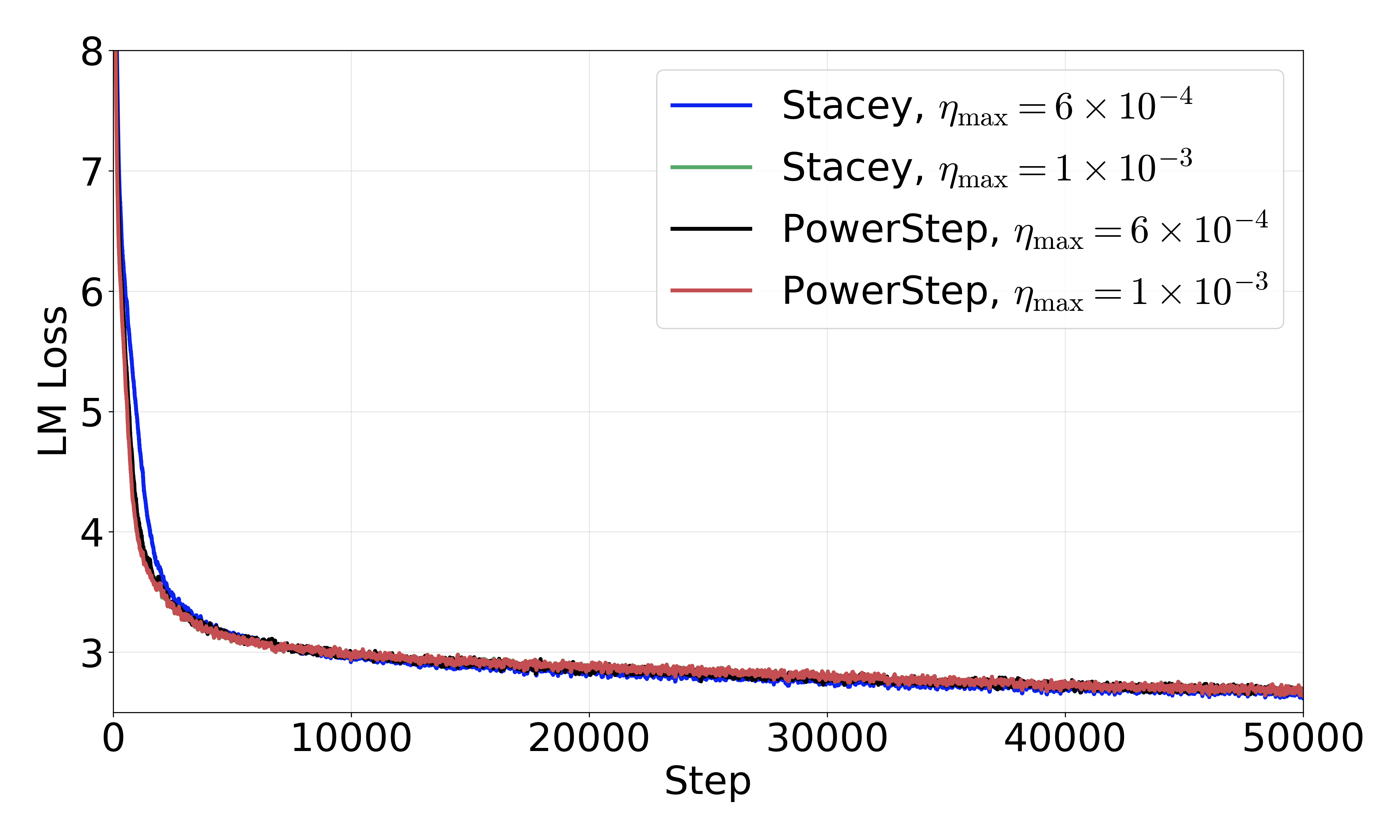}} 
    \caption{Training loss comparison between PowerStep and Stacey-$(p,2)$. The archived Medium Stacey label at $\eta_{\max}=10^{-3}$ has an unresolved provenance mismatch; see the text.}
    \label{fig:stacey}
\end{figure}

The archived file labeled as Medium Stacey at $\eta_{\max}=10^{-3}$ contains PowerStep code, so that curve's optimizer identity is not verified. Table~\ref{tab:stacey-audited} retains the independently audited validation records, which include lower losses for Stacey. The panels retain this trajectory for reference; the disputed label is not used to draw conclusions about Stacey.

\begin{table}[!htbp]
\centering\small

\begin{tabular}{lccc}\toprule
Model & $\eta_{\max}$ & Stacey & PowerStep\\\midrule
GPT-2-Small & $\text{1}\times\text{10}^{-\text{3}}$ & \textbf{2.9100}  & 2.9407\\
GPT-2-Medium & $\text{6}\times\text{10}^{-\text{4}}$ & \textbf{2.6918}  & 2.7170\\\bottomrule
\end{tabular}
\caption{Validation losses for Stacey and PowerStep after 50,000 updates at the indicated $\eta_{\max}$.}
\label{tab:stacey-audited}
\end{table}

\FloatBarrier
\subsection{Comparison with Lion}\label{sec:lion}
Lion~\citep{chen2024symbolic} also uses one momentum buffer, but computes a sign update with separate momentum coefficients for the update direction and stored state. Figure~\ref{fig:lion-training} compares Lion at three values of $\eta_{\max}$ with AdamW and PowerStep at their main-study settings. For GPT-2-Medium, Lion uses $\eta_{\max}\in\{6\times10^{-5},2\times10^{-4},6\times10^{-4}\}$; for Qwen3-8B, it uses $\eta_{\max}\in\{3\times10^{-5},10^{-4},3\times10^{-4}\}$. The lowest value follows Lion's recommendation to reduce the learning rate relative to AdamW, together with the associated optimizer settings~\citep{chen2024symbolic}.

At the recommended lower $\eta_{\max}$, Lion achieves validation loss close to AdamW and PowerStep on both models. Using the main-study $\eta_{\max}$ makes Lion unstable; the intermediate value also produces later instability despite faster initial progress on Medium. Lion becomes unstable at learning rates where AdamW and PowerStep remain stable in these runs; Appendix~\ref{app:sign-sensitivity} examines the step-size mechanism. Lion requires a smaller $\eta_{\max}$, which slows its initial progress compared with AdamW and PowerStep at their main-study settings.

\begin{figure}[!htbp]
\centering
\subfloat[GPT-2-Medium]{\includegraphics[width=.49\linewidth]{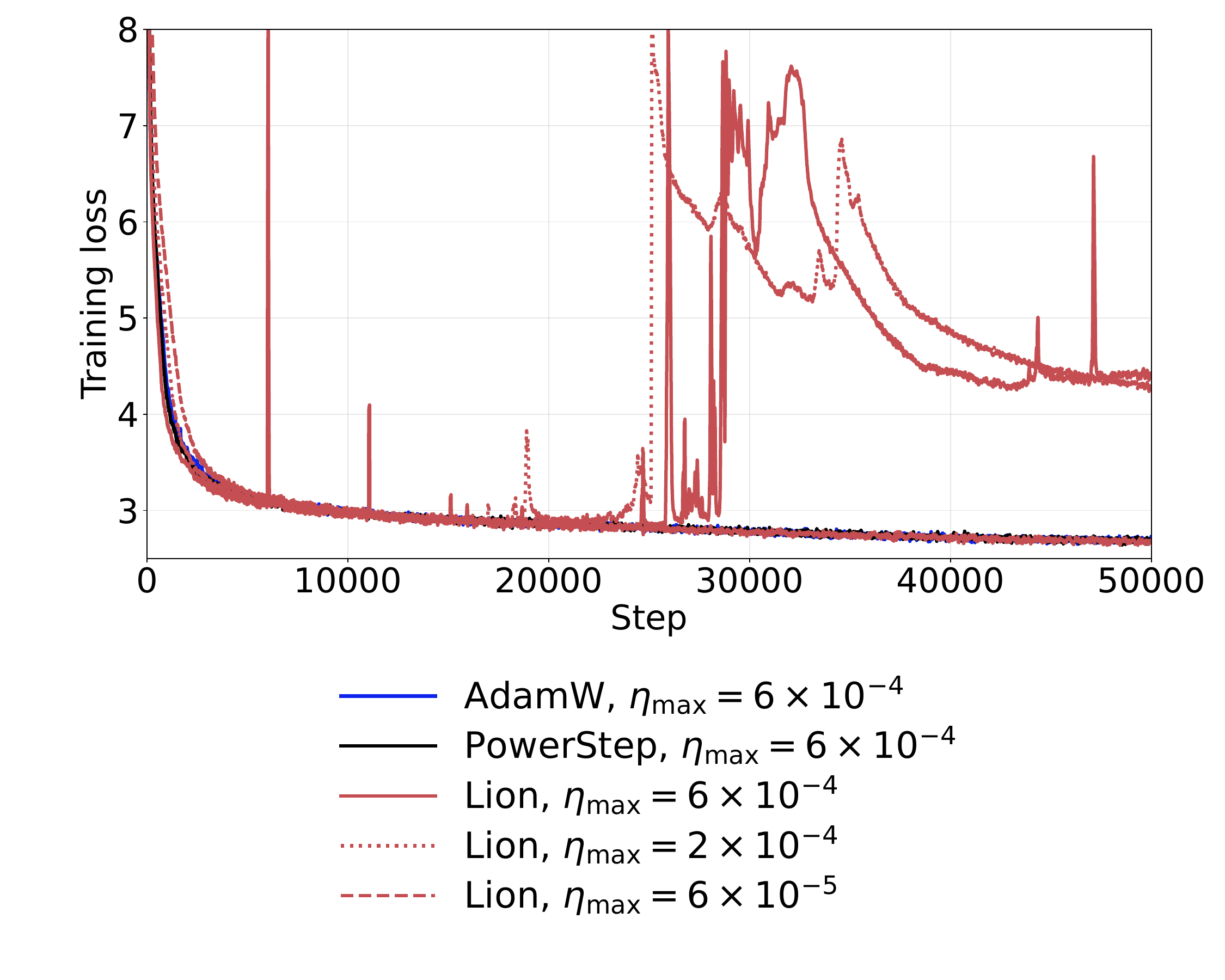}}\hfill
\subfloat[Qwen3-8B]{\includegraphics[width=.49\linewidth]{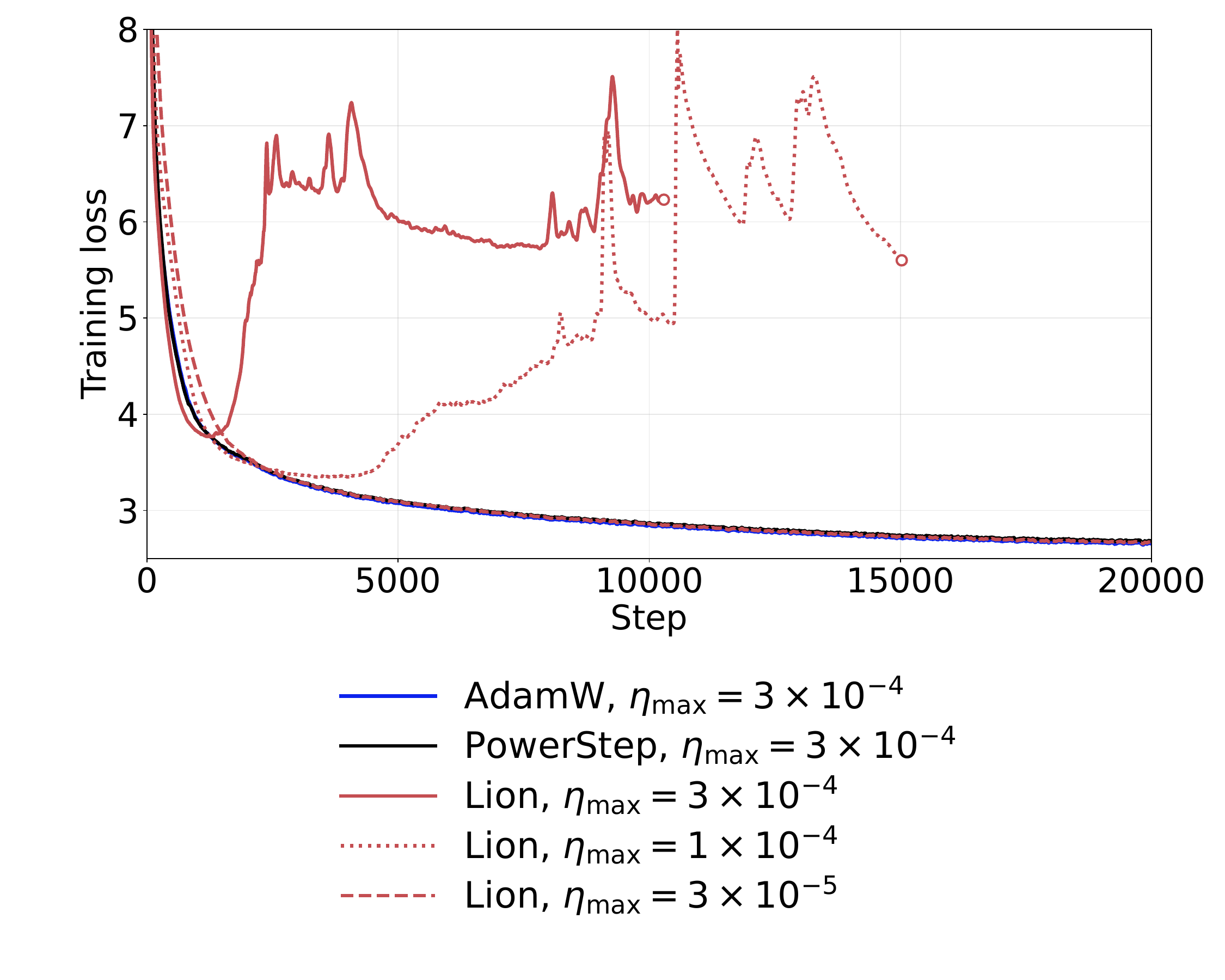}}
\caption{Training curves for Lion at three values of $\eta_{\max}$, compared with AdamW and PowerStep at the main-study settings: $\eta_{\max}=6\times10^{-4}$ on GPT-2-Medium and $3\times10^{-4}$ on Qwen3-8B. Curves show trailing 50-update averages. Colors identify optimizers; line styles distinguish Lion configurations. All GPT-2-Medium runs reach 50,000 updates. Open circles mark the recorded Qwen3-8B Lion endpoints: 10,291 updates at $3\times10^{-4}$ (runtime interruption) and 15,027 at $10^{-4}$ (partial trace).}
\label{fig:lion-training}
\end{figure}

\FloatBarrier
\subsection{PowerStep Exponent and AdamW Second-Moment Sweep}\label{sec:additional-sweep}
Figure~\ref{fig:additional-sweep} and Table~\ref{tab:sweep} vary AdamW's second-moment coefficient $\beta_2$ and PowerStep's exponent $\beta$, holding $\eta_{\max}=6\times10^{-4}$ fixed. The default rows reuse the main-study baselines. All runs reach 50,000 updates except Medium AdamW with $\beta_2=0.99$, whose last recorded evaluation is at 40,000 updates.
\begin{table}[ht]
\centering\small

\begin{tabular}{cccc}\toprule
Optimizer & Parameter & Small & Medium\\\midrule
AdamW & $\beta_\text{2}=\text{0.90}$ & 2.9686 & 2.7427\\
AdamW & $\beta_\text{2}=\text{0.95}$ (baseline) & 2.9503 & \textbf{2.7105} \\
AdamW & $\beta_\text{2}=\text{0.99}$ & 2.9397 & (collapse)\\\midrule
PowerStep & $\beta=\text{0.05}$ & \textbf{2.9350}  & 2.7291\\
PowerStep & $\beta=\text{0.10}$ (baseline) & 2.9486 & 2.7170\\
PowerStep & $\beta=\text{0.15}$ & 2.9740 & 2.7332\\\bottomrule
\end{tabular}
\caption{Parameter-sweep validation losses at $\eta_{\max}=6\times10^{-4}$. Baseline rows reuse the default configuration.}
\label{tab:sweep}
\end{table}
PowerStep remains stable across the tested exponents, while AdamW with $\beta_2=0.99$ becomes unstable on Medium. We exclude two PowerStep endpoints from the fixed-rate sweep because they use different $\eta_{\max}$: 3.0163 on Small at $10^{-2}$ and 2.7096 on Medium at $10^{-3}$. The unmatched Medium AdamW summary value 2.7258 is also excluded.

\begin{figure}[!htbp]
\centering
\subfloat[GPT-2-Small]{\includegraphics[width=.49\linewidth]{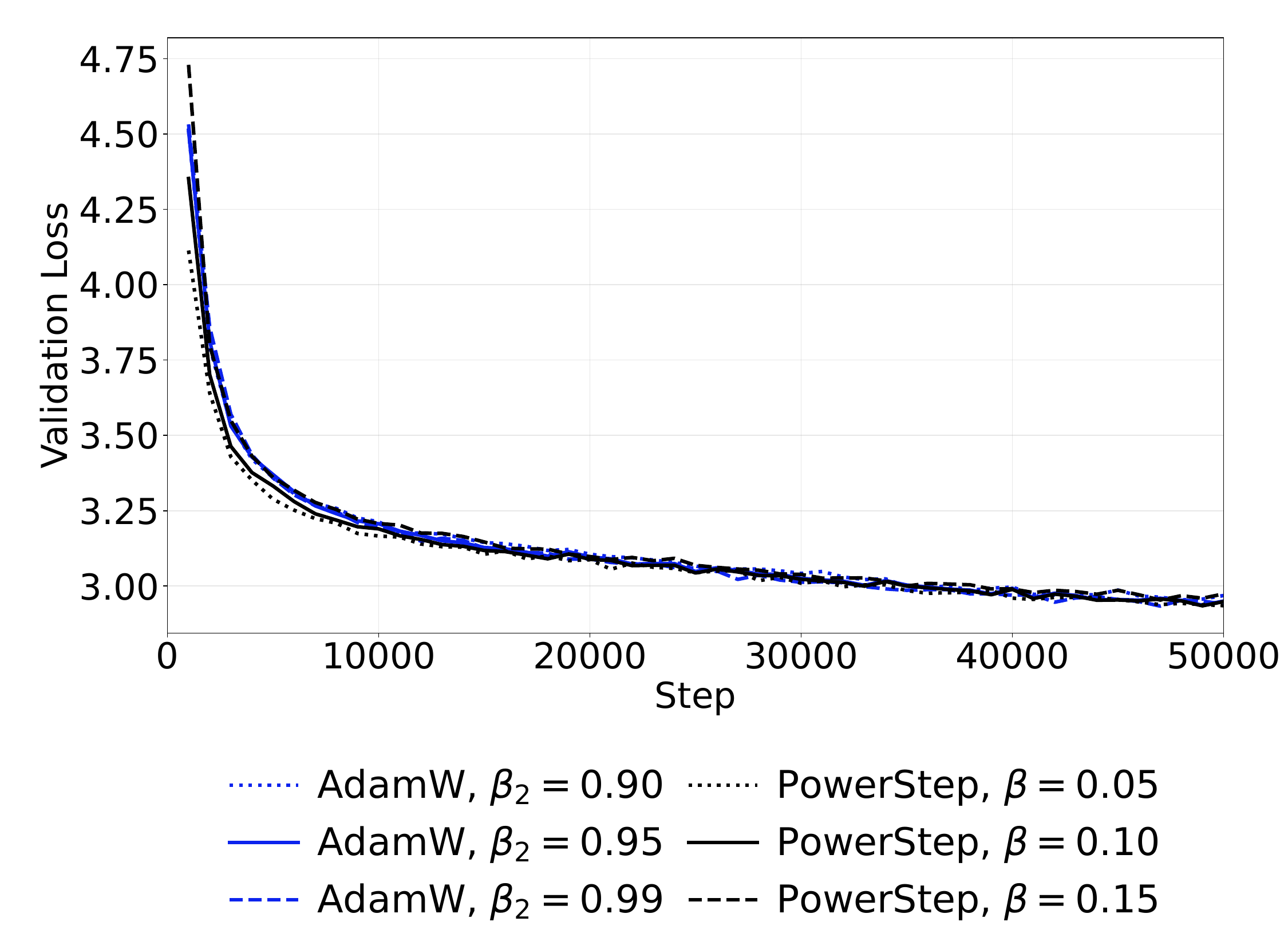}}\hfill
\subfloat[GPT-2-Medium]{\includegraphics[width=.49\linewidth]{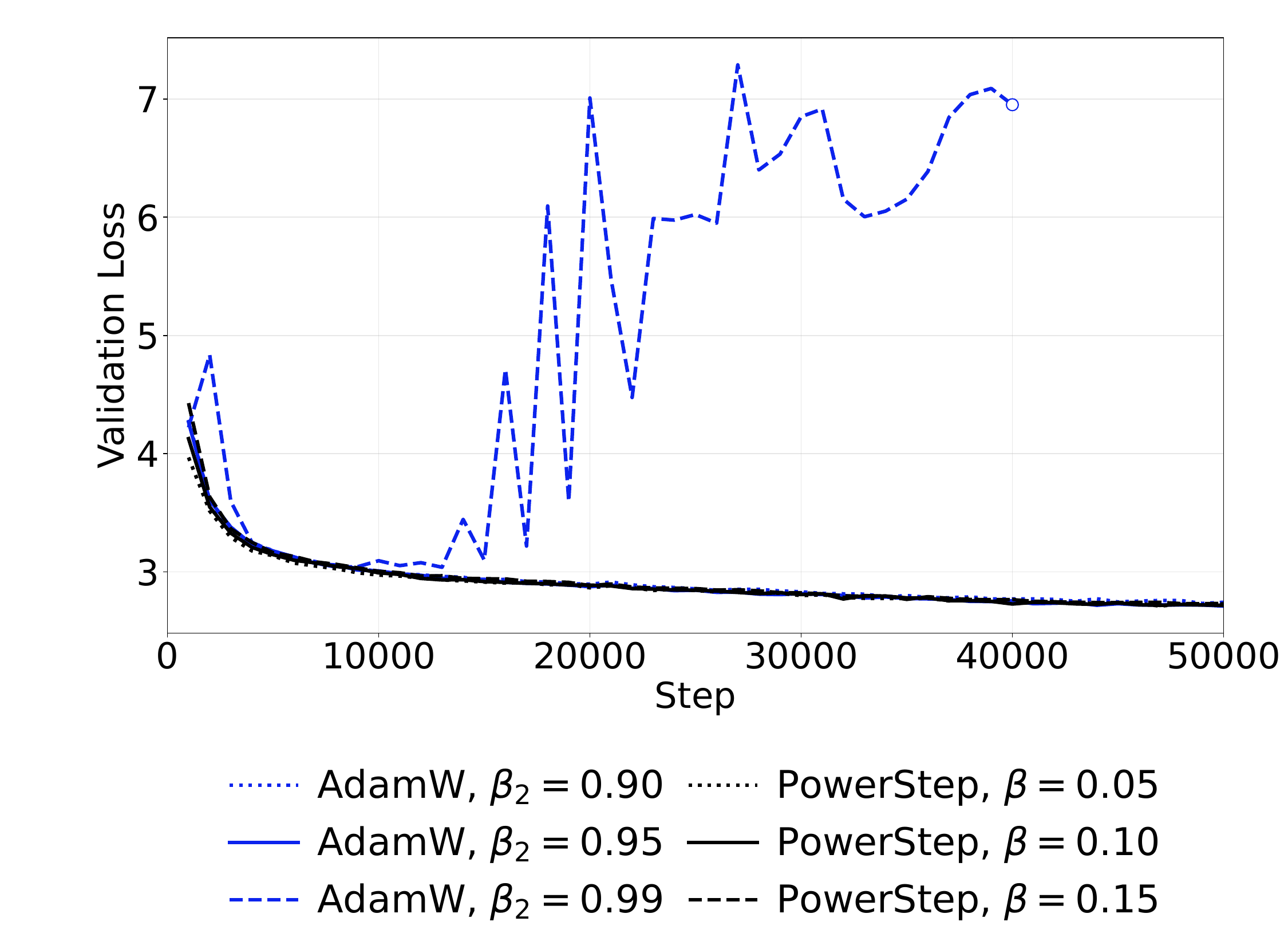}}
\caption{Validation curves for the AdamW second-moment coefficient $\beta_2$ and PowerStep exponent $\beta$ at $\eta_{\max}=6\times10^{-4}$. Solid curves use the default settings; dotted and dashed curves vary one coefficient. The open circle marks the last recorded evaluation of the unstable Medium AdamW run.}
\label{fig:additional-sweep}
\end{figure}

\subsection{Matrix Optimizers and the Role of Momentum}\label{sec:hybrids}
Figure~\ref{fig:additional-hybrids} compares Muon~\citep{jordan2024muon}, SWAN~\citep{ma2025swan}, and SinkGD~\citep{scetbon2025sinkgd} on GPT-2-Small and GPT-2-Medium. The matrix optimizers update hidden matrices, with PowerStep on the remaining parameters; we also show Muon with AdamW and the all-PowerStep configuration. SWAN normalizes and whitens the current gradient, while SinkGD applies iterative gradient normalization. Their matrix updates store no momentum. Muon instead transforms a momentum-based direction, and PowerStep applies its signed-power transform to accumulated gradients.

The Muon runs use momentum 0.95, Nesterov updates, five Newton--Schulz iterations, and update scaling by the square root of the larger matrix dimension. Both parameter groups use cosine decay from $\eta_{\max}=6\times10^{-4}$ to $\eta_{\min}=6\times10^{-5}$ without warmup. The matrix update applies no weight decay; auxiliary weight decay is 0.1. PowerStep uses $(\gamma,\beta)=(0.9,0.1)$ and AdamW uses $(\beta_1,\beta_2)=(0.9,0.95)$. SWAN and SinkGD use their separately recorded configurations. Validation is recorded every 1,000 updates.

\begin{figure}[!htbp]
\centering
\subfloat[GPT-2-Small]{\includegraphics[width=.49\linewidth]{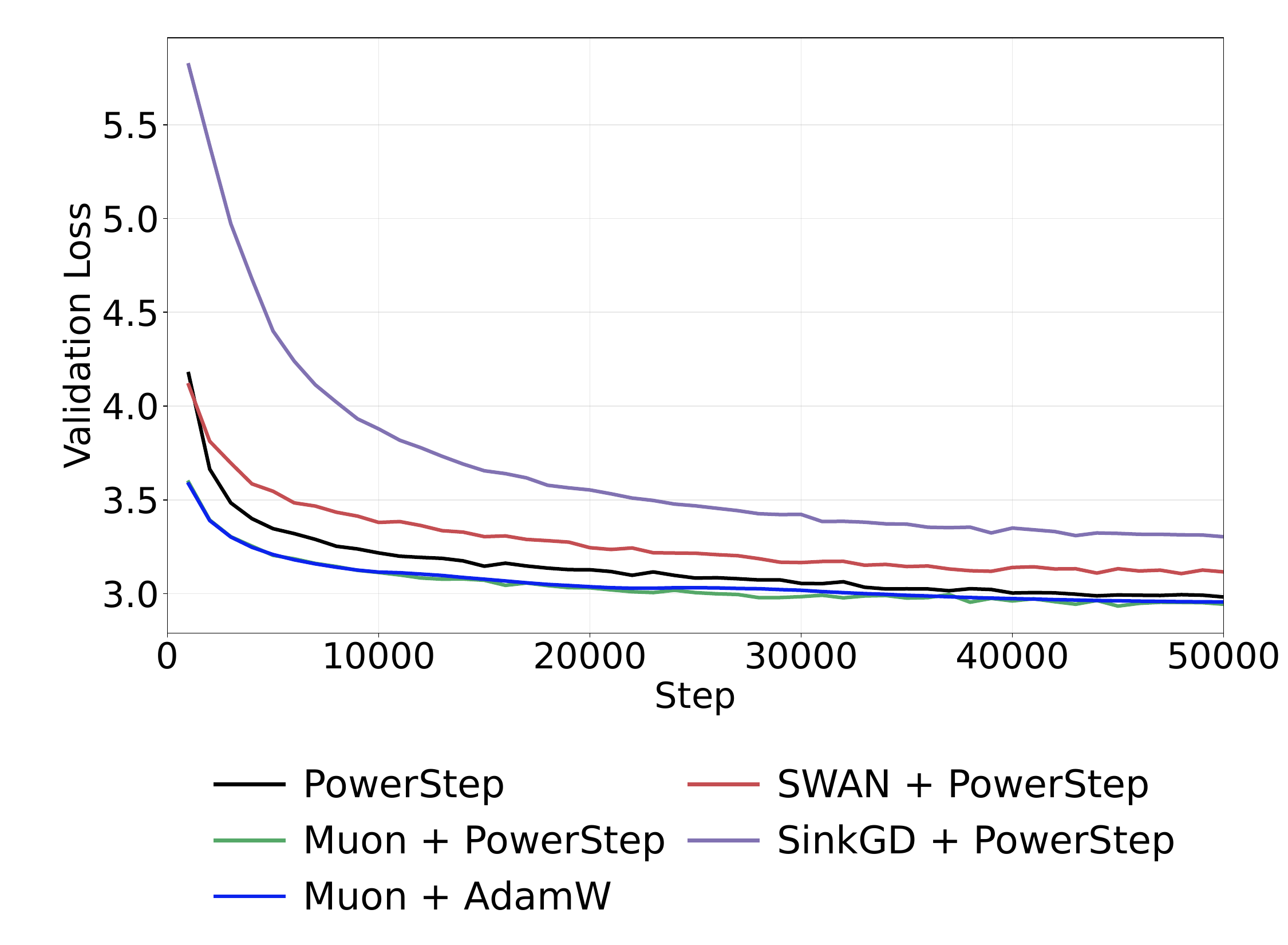}}\hfill
\subfloat[GPT-2-Medium]{\includegraphics[width=.49\linewidth]{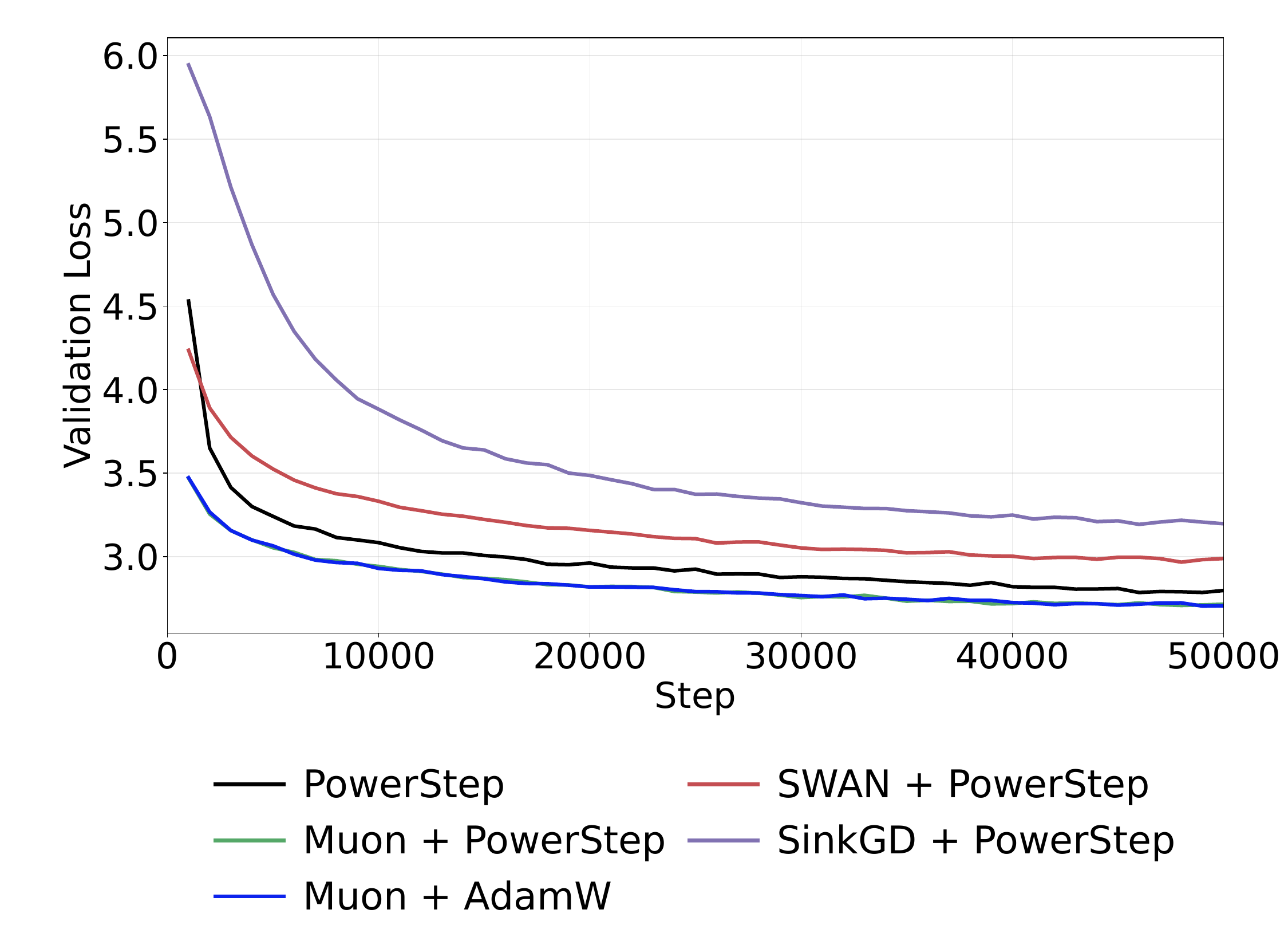}}
\caption{Validation curves for matrix-optimizer hybrids. PowerStep denotes the all-PowerStep configuration under the hybrid-study recipe; other labels identify the hidden-matrix and remaining-parameter optimizers. Muon + PowerStep and Muon + AdamW largely overlap. The Small Muon + AdamW curve is the run shown in Figure~\ref{fig:muon}.}
\label{fig:additional-hybrids}
\end{figure}

The Muon hybrids reduce validation loss faster than SWAN--PowerStep and SinkGD--PowerStep in these configurations. At 50,000 updates, SWAN--PowerStep reaches 3.1162/2.9890 and SinkGD--PowerStep reaches 3.3035/3.1973 on Small/Medium, compared with 2.9438/2.7159 for Muon--PowerStep. The all-PowerStep hybrid-study baselines reach 2.9825/2.7979. These runs change both the matrix transformation and its use of momentum, so they do not isolate the effect of momentum alone.

\paragraph{Why retain momentum?}
Momentum preserves information across minibatches, whereas a stateless update only transforms the current gradient. To illustrate the benefit, consider a locally constant gradient $\mathbf{g}$ observed with independent, zero-mean noise of covariance $\Sigma$. For $\mathbf{m}_t=\gamma\mathbf{m}_{t-1}+\mathbf{g}+\boldsymbol{\xi}_t$, the normalized momentum $(1-\gamma)\mathbf{m}_t$ has limiting mean $\mathbf{g}$ and covariance
\[
(1-\gamma)^2\sum_{k=0}^{\infty}\gamma^{2k}\Sigma
=\frac{1-\gamma}{1+\gamma}\Sigma.
\]
At $\gamma=0.9$, this is $\Sigma/19$. Temporal averaging can therefore stabilize the direction and support faster loss reduction when minibatch noise is substantial. PowerStep keeps this benefit with one buffer, then compresses coordinate magnitudes through the signed-power transform. The calculation describes the momentum before transformation, under the stated local model; it does not establish that momentum is necessary for every training problem. SWAN and SinkGD pursue a different route, using matrix structure to precondition instantaneous gradients.

Stateless matrix updates also do not make these entire hybrids stateless: PowerStep retains momentum for the remaining parameters. If hidden matrices contain $d_m$ parameters and the remaining group contains $d_r$, the \texttt{fp32} state payloads are $4d_m+4d_r$ bytes for Muon--PowerStep, $4d_m+8d_r$ for Muon--AdamW, and $4d_r$ for the SWAN/SinkGD--PowerStep hybrids. The comparison thus trades additional temporal information against further state-memory reduction.

\subsection{Supervised Fine-Tuning}\label{sec:sft}
Qwen3-4B-Base was fine-tuned with AdamW or PowerStep on a filtered high-quality, no-tool Nemotron-Math-v2 subset~\citep{nvidia2026nemotronmath}, selecting responses with boxed answers and lengths between 1,000 and 4,000 tokens. The recorded training scripts use 14,200 updates, global batch 128, microbatch 1, context length 6,144, tensor parallelism 4, eight nodes with eight devices each, \texttt{bf16}, seed 42, $\eta_{\max}=10^{-5}$, $\eta_{\min}=10^{-6}$, weight decay 0.1 and gradient clipping at 1.0. AdamW uses $(0.9,0.95)$ for its moment coefficients. Both scripts load the same base checkpoint and do not load an earlier optimizer state. Final training losses are 0.0731 and 0.0612 for AdamW and PowerStep, respectively; the final 100-update means are 0.1053 and 0.0932. The schedule is linear with no warmup.

The archived generation script samples with temperature 0.6, top-$p=0.95$ and top-$k=40$, with context length at most 32,768 and at most the remaining context length in generated tokens. Four nodes with four generation workers each produce 16 responses per question. For each model and AIME year, the raw shards contain exactly 30 questions and 480 unique question/sample pairs, without missing questions, duplicate pairs or malformed JSON records. Sampling-seed provenance and the full dataset revision/decontamination procedure remain unavailable.

We re-score both optimizers with one explicit rule: extract the \emph{last} boxed expression, accept it only if it is an integer (allowing whitespace and comma separators), and compare with the integer reference answer. Missing or unparseable answers count as incorrect. For $N$ questions with $c_i$ correct responses among 16 samples, we estimate Pass@1 as $N^{-1}\sum_{i=1}^{N}c_i/16$. Since each question has the same number of samples, this equals the fraction of correct responses among all $16N$ responses, including invalid ones in the denominator.

The AMC23 evaluation contains 40 questions and 640 responses per SFT checkpoint, with sample indices 0--15 for every question and no duplicate or conflicting records. We apply the same last-boxed-integer rule: AdamW has 359 correct responses and 34 questions with at least one correct response, versus 367 and 35 for PowerStep. There are 91 and 84 unparseable responses, respectively. The archived AMC23 set is a 40-question subset, not the full contest collection. The launch script uses the same final SFT checkpoints and decoding settings. The Pass@1 differences are small, and contamination controls remain unverified. This independent SFT experiment does not evaluate the pretraining checkpoints in Section~\ref{sec:exp}.

\FloatBarrier
\section{Quantization and Update Sensitivity}
\label{sec:quantization-analysis}
We analyze how stored-state errors affect the updates under the tested quantization scheme. Both quantizers use 128-coordinate blocks. The Dettmers configuration uses signed nonuniform codes for the first moment and unsigned codes for the second, retaining embedding and small-tensor states in \texttt{fp32}. Its local codebooks differ from those in the bitsandbytes implementation. Both GPT-2 runs use $\eta_{\max}=6\times10^{-4}$, 2,000 warmup updates and 50,000 updates with cosine decay. Figure~\ref{fig:quant} shows trailing 50-update averages of uninterrupted per-update training logs; evaluation-average losses are excluded.

\subsection{Linear Block Quantization and Memory Accounting}
Uniform quantization applies linear \texttt{int8} quantization to blocks of 128 values in every layer, including embeddings. This differs from the nonuniform dynamic recipe of \citet{dettmers2022bit}. A common symmetric nearest-rounding model is $Q(x_i)=s\,\mathrm{round}(x_i/s)$, with $s=\max_i|x_i|/127$ and $|Q(x_i)-x_i|\le s/2$; zero blocks are exact. The error is absolute relative to the block maximum, not a uniform coordinate-relative error. \texttt{int8} has integer codes, not floating-point mantissa bits. This rounding model abstracts from kernel-specific tie-breaking and scale-rounding conventions.

One \texttt{int8} state plus one \texttt{fp32} scale per 128 values uses $1+4/128=1.03125$ bytes per parameter, versus 8 bytes for two \texttt{fp32} moments: an idealized $\sim8\times$ ratio, consistent with Table~\ref{tab:memory}. This excludes parameters, gradients, master weights, activations, communication, padding and workspace. At equal \texttt{int8} precision, one state approximately halves the two-state payload, subject to metadata and tensor exemptions.

\subsection{Error Enters through the Stored Previous State}
For $0<\beta\le1$, with dequantized stored buffer $\widehat{\mathbf{m}}_{t-1}$ and full-precision accumulator $\mathbf{a}_t$, the update in Section~\ref{sec:low_precision} is
\begin{equation}\label{eq:quant-recursion}
\mathbf{a}_t=\gamma\widehat{\mathbf{m}}_{t-1}+\mathbf{g}_t,\qquad
\widehat{\mathbf{m}}_t=Q(\mathbf{a}_t),\qquad \mathbf{u}_t=\Phi_\beta(\mathbf{a}_t).
\end{equation}
Thus the current rounding error is stored for the next iteration; $\mathbf{u}_t$ is not computed from $Q(\mathbf{a}_t)$. To isolate state error, define a reference buffer on the \emph{same gradient sequence} by $\mathbf{m}_t=\gamma \mathbf{m}_{t-1}+\mathbf{g}_t$. With $\mathbf{e}_t=\widehat{\mathbf{m}}_t-\mathbf{m}_t$ and $\bm{\delta}_t=Q(\mathbf{a}_t)-\mathbf{a}_t$, starting from $\widehat{\mathbf m}_0=\mathbf m_0=\mathbf0$,
\begin{equation}
\mathbf{e}_t=\gamma \mathbf{e}_{t-1}+\bm{\delta}_t,\qquad
\mathbf{a}_t-\mathbf{m}_t=\gamma \mathbf{e}_{t-1},\qquad
\|\mathbf{e}_t\|_q\le\sum_{s=1}^t\gamma^{t-s}\|\bm{\delta}_s\|_q.
\end{equation}
The signed-power transform satisfies
\begin{equation}\label{eq:quant-holder}
\|\Phi_\beta(\mathbf{a}_t)-\Phi_\beta(\mathbf{m}_t)\|_p
\le2^{1-\beta}\gamma^\beta\|\mathbf{e}_{t-1}\|_q^\beta.
\end{equation}
The transform is continuous at zero, but neither globally bounded nor Lipschitz there for $\beta<1$. Moreover, $(1/128)^{0.1}\simeq0.616$: a H\"older bound does not prove attenuation of small absolute perturbations or preservation of coordinate signs. The same-gradient bound also does not control two training trajectories with different future gradients.

\subsection{Relative Sensitivity Away from Zero}
Relative sensitivity can be quantified when a coordinate is resolved relative to its quantization error. For $x\ne0$ and $|\delta|\le\rho|x|$ with $0\le\rho<1$, the sign is unchanged and the mean-value theorem gives
\begin{equation}\label{eq:relative-power-sensitivity}
\frac{|\Phi_\beta(x+\delta)-\Phi_\beta(x)|}{|\Phi_\beta(x)|}
\le \beta(1-\rho)^{\beta-1}\frac{|\delta|}{|x|}.
\end{equation}
Indeed, every point between $x$ and $x+\delta$ has magnitude at least $(1-\rho)|x|$, and $|\Phi_\beta'(z)|=\beta|z|^{\beta-1}$. Thus the local relative condition number is $\beta$: small relative perturbations are attenuated to first order when $\beta<1$. For example, at $\beta=0.1$ and $\rho=0.01$, the relative output error is at most $0.00101$. This is a conditional local result, not a uniform consequence of block quantization. Coordinates near zero may have large relative errors or be rounded to zero, where only the absolute H\"older bound applies. For the implemented update, $x=m_{t,i}$ and $\delta=\gamma e_{t-1,i}$, so the condition concerns accumulated previous-state error, not just one rounding step.

\subsection{Why Second-Moment Quantization Can Be Sensitive}
For $v>0$, the reciprocal factor $r(v)=1/(\sqrt v+\epsilon)$ has derivative
\begin{equation}
r'(v)=-\frac{1}{2\sqrt v(\sqrt v+\epsilon)^2}.
\end{equation}
Taylor expansion requires a small perturbation in the positive domain. Rounding $v$ to zero instead changes the reciprocal exactly from $1/(\sqrt v+\epsilon)$ to $1/\epsilon$. For example, with $v=10^{-8}$ and $\epsilon=10^{-8}$, rounding $v$ to zero increases this reciprocal factor by exactly $(\sqrt v+\epsilon)/\epsilon=10{,}001$. This example holds the numerator fixed and concerns the reciprocal factor only; it is not a predicted full-update ratio for the implementation. The condition for $\epsilon$ to dominate is $v\ll\epsilon^2$, not $v\ll\epsilon$. A first-order expansion with $|\delta^v|\gg v$ is invalid and cannot justify a quantitative amplification estimate. The update also depends on the correlated first moment, so a small denominator alone does not prove divergence. Nonuniform codebooks and selective precision can mitigate this sensitivity~\citep{dettmers2022bit,li2024memory}. Figure~\ref{fig:quant} contrasts the linear stress test with a local Dettmers baseline; uniform quantization failure is not evidence against 8-bit AdamW generally.

A nonuniform codebook improves resolution across magnitudes but does not guarantee small relative errors for every coordinate. Rounding errors in stored moments enter subsequent recurrences. If the second moment is underestimated without a corresponding decrease in the first, the update can be amplified. Q-Adam-mini~\citep{han2025qadam} provides another example of Adam-state quantization instability: it reports embedding weight-norm instability from first-moment quantization and mitigates it with stochastic rounding. This supports the need for care when quantizing Adam, but does not identify the cause of our Medium loss spikes, since embedding states remain in \texttt{fp32} in the evaluated Dettmers configuration.

Independent low-bit Adam studies identify related failure modes. \citet[][Table~3]{dettmers2022bit} report 90\% unstable runs for unblocked linear 8-bit Adam on a 209M language model, versus 0\% for their block-wise dynamic scheme. These 209M percentages summarize runs across hyperparameter settings. Their nonuniform but unblocked 1.3B configuration without the stable embedding layer is unstable in all five seeds (100\%); adding that layer reduces the reported rate to 80\%. At 4 bits, \citet{li2024memory} identify the second-moment zero-point problem: rounding small positive states to zero can greatly distort the inverse-square-root factor. Their GPT-2-Medium E2E-NLG fine-tuning ablation (Table~1) reports 66\% unstable runs for B128/DE and 0\% for B128/DE-0; neither row uses stable embeddings, stochastic rounding or factorization. DE-0 excludes zero from the dynamic-exponent codebook. The latter result concerns 4-bit states and supports the mechanism rather than an 8-bit failure rate. Together these studies establish that Adam-state quantization can require specific stabilization measures; they do not imply that all 8-bit Adam implementations are unstable.

\paragraph{Location of the observed Medium instability.}
After 10,000 updates, the first raw training loss above 4 occurs at completed update 15,998 (loss 4.1295). Validation rises at 16,000 updates, temporarily recovers by 19,000, and deteriorates again after 20,000. These are finite loss spikes; the run reaches its full 50,000-update budget. The implementation forms
$\widetilde{\mathbf m}_t=\beta_1\widehat{\mathbf m}_{t-1}+(1-\beta_1)\mathbf g_t$
and
$\widetilde{\mathbf v}_t=\beta_2\widehat{\mathbf v}_{t-1}+(1-\beta_2)\mathbf g_t^{\odot2}$
in \texttt{fp32} and uses these tensors, with bias correction, in the current update. Quantized copies are stored for the next iteration. Thus a second moment rounded down at one step can affect a later denominator; the current squared-gradient contribution can partly replenish it. First- and second-moment errors need not cancel in their ratio. Identifying a causal trigger would require layerwise moment-error and update-norm records around the first excursion; the available loss logs alone do not provide that localization.

\subsection{Step-Size Sensitivity of Sign and Signed-Power Updates}
\label{app:sign-sensitivity}
Both sign updates and PowerStep use the sign operator to choose direction. PowerStep additionally retains magnitude information. For a nonzero scalar momentum $m$, their step magnitudes are $\eta_t$ and $\eta_t|m|^\beta$, respectively. A sign update takes a full coordinate step even when the momentum is arbitrarily small. PowerStep's step vanishes as $m\to0$ for every $\beta>0$. Changing the momentum from $+\delta$ to $-\delta$, with $\delta>0$, changes the update by
\begin{equation}
2\eta_t \quad\text{for sign},\qquad
2\eta_t\delta^\beta \quad\text{for signed power}.
\end{equation}
Although $\operatorname{sign}(m)$ is discontinuous at zero, its product with $|m|^\beta$ is continuous there. This reduces the update jump caused by sign changes of small momentum coordinates. At $\beta=0.1$, however, the attenuation is slow: $(10^{-6})^{0.1}\approx0.251$.

A scalar quadratic isolates the overshoot mechanism. Let $f(x)=ax^2/2$ with $a>0$, and consider exact gradients without momentum or weight decay. For $x_t\ne0$, sign descent gives $x_{t+1}=x_t-\eta_t\operatorname{sign}(x_t)$, so
\begin{equation}
f(x_{t+1})-f(x_t)=-a\eta_t|x_t|+\tfrac12a\eta_t^2<0
\quad\Longleftrightarrow\quad 0<\eta_t<2|x_t|.
\end{equation}
The descent threshold shrinks near the minimum, while the step magnitude remains fixed. With a constant positive step size, the iterates generally enter a two-point cycle around zero unless they reach zero exactly. For the signed-power update $x_{t+1}=x_t-\eta_t\operatorname{sign}(x_t)|ax_t|^\beta$, the same calculation yields
\begin{align}
f(x_{t+1})-f(x_t)
&=-\eta_t a^{1+\beta}|x_t|^{1+\beta}
+\tfrac12\eta_t^2a^{1+2\beta}|x_t|^{2\beta},\\
f(x_{t+1})<f(x_t)
&\quad\Longleftrightarrow\quad 0<\eta_t a^\beta<2|x_t|^{1-\beta}.
\end{align}
Signed power lets the absolute step shrink with the gradient, but for $0<\beta<1$ a fixed step size can still produce oscillation near the minimum. Momentum can also lag behind the current gradient. Thus PowerStep softens the fixed-step and zero-crossing mechanisms rather than eliminating overshoot or guaranteeing stability.

\citet{chen2024symbolic} likewise attribute Lion's smaller recommended learning rate to the larger update norm produced by the sign operation. This scalar analysis explains why retaining magnitude information can help, but does not establish a universal stability ordering against AdamW, whose normalized update can also become sign-like. Requiring a smaller learning rate is distinct from having a narrower stable learning-rate range; Figure~\ref{fig:lion-training} directly illustrates the former under the tested configurations.

\FloatBarrier
\section{Complete Convergence Proofs}
\label{app:proofs}
Throughout this appendix, $0<\beta\le1$, $q=1+\beta$, $p=q/\beta$, $h=1-\gamma$, and the assumptions are those of Section~\ref{sec:convergence}. All vector norms below are explicitly indicated. The update is exact, with no weight decay, clipping, or quantization. Expectations refer to the stochastic oracle, and the horizon-dependent step size is constant during each run. Write $\bar{\mathbf{g}}_t=\nabla f(\bm{\theta}_{t-1})$, $\mathbf{u}_t=\Phi_\beta(\mathbf{m}_t)$, and
\begin{equation}\label{eq:buffer-decomp}
\bar{\mathbf{G}}_t=\sum_{k=0}^{t-1}\gamma^k\bar{\mathbf{g}}_{t-k},\qquad
\mathbf{N}_t=\sum_{k=0}^{t-1}\gamma^k\bm{\xi}_{t-k},\qquad
\mathbf{m}_t=\bar{\mathbf{G}}_t+\mathbf{N}_t,\qquad \mathbf{r}_t=\bar{\mathbf{g}}_t-h\bar{\mathbf{G}}_t.
\end{equation}
Here $\bar{\mathbf{G}}_t$ is an accumulated true-gradient component, not a conditional expectation of $\mathbf{m}_t$ at the current iterate. Past noise is correlated with later iterates; no independence between $\mathbf{N}_t$ and $\bar{\mathbf{G}}_t$ is assumed.

\paragraph{Explicit constants in Theorem~\ref{thm:convergence}.}
Let $\Delta_0=f(\bm{\theta}_0)-f^*$. Define
\begin{align}
M&=\left(\frac{G+\sigma}{h}\right)^{2\beta},
&D&=\frac{L_p\sqrt M\,\gamma}{h},\nonumber\\
A&=2^{-q}h^{-\beta},
&B&=\sqrt M+\tfrac12h^{-\beta}(2G)^\beta,\label{eq:theory-constants}\\
K_\beta&=\frac{\beta}{2}\left(\frac{2^{2-\beta}}q\right)^{q/\beta}.\nonumber
\end{align}
Theorem~\ref{thm:convergence} uses
\begin{equation}
C_1=\frac1A\left[\frac{\Delta_0}{c}+c\left(\frac{L_pM}{2}+BD\right)\right],\qquad
C_2=\frac{BG\gamma}{Ah},\qquad C_\beta=2^qK_\beta.
\end{equation}
Thus $R_T=C_1/\sqrt T+C_2/T$ retains the initialization term and all parameter dependencies.

\paragraph{Proof idea and scope.}
Write $\mathbf{m}_t=\bar{\mathbf{G}}_t+\mathbf{N}_t$, where $\bar{\mathbf{G}}_t=\sum_{k=0}^{t-1}\gamma^k\bar{\mathbf{g}}_{t-k}$ and $\mathbf{N}_t=\sum_{k=0}^{t-1}\gamma^k\bm{\xi}_{t-k}$. Align the update through $\bar{\mathbf{G}}_t$ and control $\bar{\mathbf{g}}_t-h\bar{\mathbf{G}}_t$ in mean square. This makes the drift contribution linear in $\eta$, rather than $\eta^\beta$. A coordinatewise H\"older bound and Young's inequality control the noise without removing its residual. The following lemmas account for initialization and give explicit constants. The constant horizon-dependent step avoids the logarithm in $\sum_t(c/\sqrt t)^2$; it differs from the experimental warmup/cosine schedule.

When $\sigma=0$, the bound gives $O(T^{-1/2})$ stationarity. A vanishing average noise moment also removes the residual. The corollaries below state sufficient conditions for growing batches and a small relative-noise condition. These are additional assumptions, not properties established for the language-model experiments. We make no minimax-optimality claim: a neighborhood bound is weaker than exact stationarity, and increasing the batch changes the number of oracle calls.

\paragraph{Dimension and the choice of $\beta$.}
There is no explicit $d$ in~\eqref{eq:corrected-rate} when $L_p,G,\sigma$ are fixed in the native norms. This does not imply dimension independence at fixed Euclidean constants. If those constants are $L_2,G_2,\sigma_2$, valid choices are
\begin{equation}\label{eq:dimension-conversion}
L_p=d^{(1-\beta)/(1+\beta)}L_2,\quad
G=d^{1/q-1/2}G_2,\quad
\sigma^q=d^{1-q/2}\sigma_2^q.
\end{equation}
For example, at $\beta=0.1$ and $d=10^5$, the Euclidean-to-native conversion gives $L_p=d^{9/11}L_2\approx1.23\times10^4L_2$. This factor concerns conversion from a fixed Euclidean smoothness constant, not an additional factor when native-norm constants are assumed directly. At the experimental exponent $\beta=0.1$, $K_\beta\approx3.429\times10^4$ and $C_\beta=2^{1.1}K_\beta\approx7.350\times10^4$. The resulting noise-residual bound is therefore a qualitative guarantee, not a quantitative prediction of training performance. More generally, $C_\beta$ becomes loose as $\beta\to0$; the theorem neither covers $\beta=0$ nor predicts the empirical preference for $\beta=0.1$. Even at $\beta=1$, the residual in this generic bound need not be tight. Removing bounded gradients and proving sharper noise-dependent guarantees remain open here.

\subsection{Transform and Moment Bounds}
\begin{lemma}[Dual geometry and scalar perturbations]\label{lemma:holder}
For $\mathbf{x},\mathbf{y},\mathbf{m}\in\mathbb R^d$,
\begin{equation}
\langle \mathbf{m},\Phi_\beta(\mathbf{m})\rangle=\|\mathbf{m}\|_q^q,\quad
\|\Phi_\beta(\mathbf{m})\|_p=\|\mathbf{m}\|_q^\beta,\quad
\|\Phi_\beta(\mathbf{x})-\Phi_\beta(\mathbf{y})\|_p\le2^{1-\beta}\|\mathbf{x}-\mathbf{y}\|_q^\beta.
\end{equation}
Moreover, for real $a,n$,
\begin{equation}\label{eq:young-explicit}
2^{1-\beta}|a||n|^\beta\le\tfrac12|a|^q+K_\beta|n|^q,
\qquad K_\beta=\frac\beta2\left(\frac{2^{2-\beta}}q\right)^{q/\beta}.
\end{equation}
\end{lemma}
\begin{proof}
The identities follow coordinatewise from $\beta p=q$. For same-sign scalars, subadditivity of $z^\beta$ gives $|\Phi_\beta(x)-\Phi_\beta(y)|\le|x-y|^\beta$. For opposite signs, concavity gives $|x|^\beta+|y|^\beta\le2^{1-\beta}(|x|+|y|)^\beta$. Raising the scalar bound to $p$, summing, and taking the $p$-th root proves the vector bound. Finally maximize $2^{1-\beta}a n^\beta-a^q/2$ over $a\ge0$ for fixed $n\ge0$. For $n>0$ the maximizer is $a=(2^{2-\beta}/q)^{1/\beta}n$, giving~\eqref{eq:young-explicit}; the zero cases are immediate.
\end{proof}

\begin{lemma}[Uniform update moment and accumulated noise]\label{lemma:momentum_bound}
Let $\nu_t=\mathbb E\|\bm{\xi}_t\|_q^q\le\sigma^q$. Then
\begin{align}
\mathbb E\|\mathbf{u}_t\|_p^2&\le M=\left(\frac{G+\sigma}{h}\right)^{2\beta},\label{eq:update-moment}\\
\mathbb E\|\mathbf{N}_t\|_q^q&\le h^{1-q}\sum_{k=0}^{t-1}\gamma^k\nu_{t-k},\qquad
\sum_{t=1}^T\mathbb E\|\mathbf{N}_t\|_q^q\le h^{-q}\sum_{t=1}^T\nu_t.\label{eq:noise-moment}
\end{align}
\end{lemma}
\begin{proof}
Minkowski's inequality in $L^q$ gives $(\mathbb E\|\mathbf{g}_t\|_q^q)^{1/q}\le G+\sigma$ and
$(\mathbb E\|\mathbf{m}_t\|_q^q)^{1/q}\le\sum_{k=0}^{t-1}\gamma^k(G+\sigma)\le(G+\sigma)/h$.
Since $2\beta/q\le1$, Jensen yields
$\mathbb E\|\mathbf{u}_t\|_p^2=\mathbb E(\|\mathbf{m}_t\|_q^q)^{2\beta/q}\le(\mathbb E\|\mathbf{m}_t\|_q^q)^{2\beta/q}\le M$.
This does not require a second moment of the untransformed noise.
For nonnegative weights $w_k$, convexity gives
$\|\sum_k w_k\mathbf{x}_k\|_q^q\le(\sum_k w_k)^{q-1}\sum_k w_k\|\mathbf{x}_k\|_q^q$.
Set $w_k=\gamma^k$ and take expectations. Summing and interchanging the finite sums bounds the remaining geometric sum by $1/h$. No orthogonality or independence of the noise terms is used. Neither this estimate nor the subsequent alignment argument requires a zero conditional mean; oracle bias is included in the noise moment and its residual.
\end{proof}

\subsection{Drift and Alignment}
\begin{lemma}[Initialization and gradient drift]\label{lemma:drift}
With constant step size $\eta$, let $D=L_p\sqrt M\gamma/h$. Then
\begin{equation}\label{eq:drift-rms}
\left(\mathbb E\|\mathbf{r}_t\|_q^2\right)^{1/2}\le G\gamma^t+D\eta=:R_t^{\rm drift},
\qquad \|\mathbf{r}_t\|_q\le2G\ \text{almost surely}.
\end{equation}
Throughout, $R_t^{\rm drift}$ denotes this deterministic RMS upper bound; it also bounds $\mathbb E\|\mathbf r_t\|_q$ by Cauchy--Schwarz. Consequently,
\begin{equation}\label{eq:drift-products}
\mathbb E[\|\mathbf{r}_t\|_q\|\mathbf{u}_t\|_p]\le\sqrt M R_t^{\rm drift},
\qquad \mathbb E\|\mathbf{r}_t\|_q^q\le(2G)^{q-1}R_t^{\rm drift}.
\end{equation}
\end{lemma}
\begin{proof}
From~\eqref{eq:buffer-decomp},
\begin{equation}
\mathbf{r}_t=\gamma^t\bar{\mathbf{g}}_t+h\sum_{k=1}^{t-1}\gamma^k(\bar{\mathbf{g}}_t-\bar{\mathbf{g}}_{t-k}).
\end{equation}
The first term is bounded by $G\gamma^t$. By the gradient Lipschitz assumption,
$\|\bar{\mathbf{g}}_t-\bar{\mathbf{g}}_{t-k}\|_q\le L_p\eta\sum_{j=t-k}^{t-1}\|\mathbf{u}_j\|_p$.
Minkowski's inequality in $L^2$ and~\eqref{eq:update-moment} therefore give
\begin{equation}
(\mathbb E\|\mathbf{r}_t\|_q^2)^{1/2}
\le G\gamma^t+hL_p\eta\sqrt M\sum_{k=1}^{t-1}k\gamma^k
\le G\gamma^t+\frac{L_p\eta\sqrt M\gamma}{h}.
\end{equation}
Also $\|\bar{\mathbf{g}}_t\|_q\le G$ and $h\|\bar{\mathbf{G}}_t\|_q\le G(1-\gamma^t)$, proving the pointwise $2G$ bound. Cauchy--Schwarz gives the first product bound without assuming independence. The second uses $\|\mathbf{r}_t\|_q^q\le(2G)^{q-1}\|\mathbf{r}_t\|_q$ followed by $\mathbb E\|\mathbf r_t\|_q\le(\mathbb E\|\mathbf r_t\|_q^2)^{1/2}\le R_t^{\rm drift}$. In particular, we never substitute the moment bound $M$ as a pointwise bound on $\|\mathbf{u}_t\|_p^2$.
\end{proof}

\begin{lemma}[Alignment with the current gradient]\label{lemma:alignment}
Let $A,B$ be defined in~\eqref{eq:theory-constants}. Then
\begin{equation}\label{eq:alignment-correct}
\mathbb E\langle\bar{\mathbf{g}}_t,\mathbf{u}_t\rangle
\ge A\mathbb E\|\bar{\mathbf{g}}_t\|_q^q
-hK_\beta\mathbb E\|\mathbf{N}_t\|_q^q-BR_t^{\rm drift}.
\end{equation}
\end{lemma}
\begin{proof}
The scalar perturbation bound followed by~\eqref{eq:young-explicit} gives
\begin{align}
\langle\bar{\mathbf{G}}_t,\Phi_\beta(\bar{\mathbf{G}}_t+\mathbf{N}_t)\rangle
&\ge\|\bar{\mathbf{G}}_t\|_q^q-2^{1-\beta}\sum_i|\bar G_{t,i}||N_{t,i}|^\beta\nonumber\\
&\ge\tfrac12\|\bar{\mathbf{G}}_t\|_q^q-K_\beta\|\mathbf{N}_t\|_q^q.
\end{align}
Since $\bar{\mathbf{g}}_t=h\bar{\mathbf{G}}_t+\mathbf{r}_t$, dual-pair H\"older and Lemma~\ref{lemma:drift} imply
\begin{equation}
\mathbb E\langle\bar{\mathbf{g}}_t,\mathbf{u}_t\rangle
\ge\tfrac h2\mathbb E\|\bar{\mathbf{G}}_t\|_q^q-hK_\beta\mathbb E\|\mathbf{N}_t\|_q^q-\sqrt M R_t^{\rm drift}.
\end{equation}
Finally $\|\bar{\mathbf{g}}_t\|_q^q\le2^{q-1}(h^q\|\bar{\mathbf{G}}_t\|_q^q+\|\mathbf{r}_t\|_q^q)$, hence
\begin{equation}
\tfrac h2\mathbb E\|\bar{\mathbf{G}}_t\|_q^q
\ge2^{-q}h^{1-q}\mathbb E\|\bar{\mathbf{g}}_t\|_q^q
-\tfrac12h^{1-q}(2G)^{q-1}R_t^{\rm drift}.
\end{equation}
Substituting $q-1=\beta$ proves the result.
\end{proof}

\subsection{Proof of Theorem~\ref{thm:convergence}}
Integrating the gradient Lipschitz assumption along a line segment gives
$f(\mathbf{y})\le f(\mathbf{x})+\langle\nabla f(\mathbf{x}),\mathbf{y}-\mathbf{x}\rangle+(L_p/2)\|\mathbf{y}-\mathbf{x}\|_p^2$.
The update and the moment bound imply
\begin{equation}
\mathbb E f(\bm{\theta}_t)\le\mathbb E f(\bm{\theta}_{t-1})
-\eta\mathbb E\langle\bar{\mathbf{g}}_t,\mathbf{u}_t\rangle+\tfrac{L_pM}2\eta^2.
\end{equation}
The bounded-gradient assumption and finite update moments ensure integrability at each finite time. Summing and using $f\ge f^*$ yields
\begin{equation}
\eta\sum_{t=1}^T\mathbb E\langle\bar{\mathbf{g}}_t,\mathbf{u}_t\rangle\le\Delta_0+\tfrac{L_pM}2T\eta^2.
\end{equation}
Insert~\eqref{eq:alignment-correct}, use~\eqref{eq:noise-moment}, and note $\sum_{t=1}^TR_t^{\rm drift}\le TD\eta+G\gamma/h$:
\begin{equation}\label{eq:master-bound}
A\frac1T\sum_{t=1}^T\mathbb E\|\bar{\mathbf{g}}_t\|_q^q
\le\frac{\Delta_0}{\eta T}+\left(\tfrac{L_pM}2+BD\right)\eta
+\frac{BG\gamma}{hT}+K_\beta h^{1-q}\frac1T\sum_{t=1}^T\nu_t.
\end{equation}
Since $h^{1-q}/A=2^q$, substituting $\eta=c/\sqrt T$ gives~\eqref{eq:corrected-rate}. For the squared Euclidean gradient, $|\bar g_{t,i}|\le G$ implies pointwise
\begin{equation}\label{eq:correct-norm-conversion}
\|\bar{\mathbf{g}}_t\|_2^2=\sum_i|\bar g_{t,i}|^2
\le G^{2-q}\sum_i|\bar g_{t,i}|^q=G^{1-\beta}\|\bar{\mathbf{g}}_t\|_q^q.
\end{equation}
Take expectations and use that the minimum is at most the average. This proves the theorem.\hfill$\square$

\subsection{Conditions That Remove the Residual}
\begin{corollary}[Deterministic and vanishing noise]\label{cor:vanishing}
In Theorem~\ref{thm:convergence}, $\sigma=0$ gives $O(T^{-1/2})$ for the average squared gradient. More generally, if $\nu_t\le\bar\nu_t\le\sigma^q$ for a fixed deterministic sequence with $T^{-1}\sum_{t=1}^T\bar\nu_t\to0$, the bound tends to zero as the horizon grows. This is a guarantee for a family of horizon-dependent runs, not an almost-sure convergence claim for a single infinite run.
\end{corollary}
\begin{proof}
Use the actual moments $\nu_t$ in~\eqref{eq:corrected-rate}; the uniform envelope keeps $M,D,B$ fixed. Both terms on the right tend to zero.
\end{proof}

\begin{corollary}[Relative-noise condition]\label{cor:growth}
Replace the noise assumption by
$\mathbb E[\|\bm{\xi}_t\|_q^q\mid\mathcal F_{t-1}]\le\rho^q\|\bar{\mathbf{g}}_t\|_q^q$.
If $2^qK_\beta\rho^q\le1/2$, then
\begin{equation}
\frac1T\sum_{t=1}^T\mathbb E\|\bar{\mathbf{g}}_t\|_q^q\le2R_T=O(T^{-1/2}),
\end{equation}
where $R_T$ uses $\sigma=\rho G$. The same rate holds for the average squared Euclidean gradient after multiplication by $G^{1-\beta}$.
\end{corollary}
\begin{proof}
Bounded gradients imply the uniform envelope $\sigma=\rho G$, so every lemma still applies. In~\eqref{eq:corrected-rate}, substitute $\nu_t\le\rho^q\mathbb E\|\bar{\mathbf{g}}_t\|_q^q$ and move this term to the left. The assumed coefficient is at most $1/2$.
\end{proof}
This is an explicit, restrictive sufficient condition, especially for small $\beta$. It is not implied by unbiasedness or ordinary bounded variance, and it has not been checked in our experiments.

\begin{corollary}[Growing independent mini-batches]\label{cor:batch}
Suppose the oracle averages $b_t$ conditionally independent, unbiased per-sample gradients with conditional Euclidean noise variance at most $v^2$. In addition to Assumption~\ref{assum:regularity}, this gives
\begin{equation}
\nu_t\le d^{1-q/2}v^q b_t^{-q/2}.
\end{equation}
Any $b_t\to\infty$ removes the residual in the horizon limit. If $b_t=\lceil a t\rceil$ for $a>0$, the squared-gradient bound is $O(T^{-1/2})$ for every fixed $\beta>0$.
\end{corollary}
\begin{proof}
Conditional independence gives $\mathbb E[\|\bm{\xi}_t\|_2^2\mid\mathcal F_{t-1}]\le v^2/b_t$. The norm inequality $\|\mathbf{x}\|_q\le d^{1/q-1/2}\|\mathbf{x}\|_2$ and concavity of $z^{q/2}$ give the stated moment bound. Ces\`aro averaging proves the first assertion. For $1<q<2$, $T^{-1}\sum_{t=1}^Tt^{-q/2}=O(T^{-q/2})=O(T^{-1/2})$; for $q=2$, it is $O(\log(T)/T)=O(T^{-1/2})$. Apply Theorem~\ref{thm:convergence} with the uniform envelope $\sigma^q=d^{1-q/2}v^q$.
\end{proof}
A mere finite $q$-moment does not imply the $b_t^{-q/2}$ scaling; the per-sample second-moment assumption above supplies it. Linear batch growth uses $N=\Theta(T^2)$ stochastic gradient evaluations, yielding only $O(N^{-1/4})$ for the squared-gradient bound derived here. We do not call this sample-optimal.

\subsection{Interpretation and Limits}
For $q\le2\le p$, norm equivalence gives the valid conversions in~\eqref{eq:dimension-conversion}: applying $\|\mathbf{z}\|_q\le d^{1/q-1/2}\|\mathbf{z}\|_2$ to gradient differences and $\|\mathbf{x}-\mathbf{y}\|_2\le d^{1/2-1/p}\|\mathbf{x}-\mathbf{y}\|_p$ to displacements yields $L_p=d^{1/q-1/p}L_2$. The bounded-gradient and noise conversions follow in the same way, with Jensen for the noise. Thus native notation removes explicit conversion factors from the proof, but does not remove possible dimensional scaling of the assumptions.

The residual is an upper-bound term, not a matching lower bound or a claim that every trajectory stalls. Unbiased noise alone does not commute with the nonlinear transform: for example, a scalar noise taking $2$ with probability $1/3$ and $-1$ with probability $2/3$ has zero mean but $\mathbb E\Phi_\beta(\xi)=(2^\beta-2)/3\ne0$ for $\beta<1$. Additional distributional or variance-reduction assumptions therefore matter. This example explains why exact alignment cannot simply be inferred from unbiasedness; it is not a counterexample to every fixed-momentum run.

Our proof uses geometric-sum convexity, not a second-moment martingale identity for $q<2$. It also retains the initialization term $G\gamma^t$. It supplies neither a guarantee for quantization, clipping, and decoupled weight decay nor a proof for scheduled momentum or the practical cosine schedule. Those extensions require separate analyses.

\end{document}